\documentclass[10pt,reqno]{amsart}
\usepackage{graphicx} 
\usepackage{xcolor}
\usepackage{amsmath,amsfonts,amssymb}
\usepackage{amsthm}
\usepackage{mathtools}
\usepackage{hyperref}
\usepackage[capitalise]{cleveref}
\usepackage{fullpage}
\usepackage{algorithm}
\usepackage{algpseudocode}
\usepackage{siunitx}
\usepackage{orcidlink}
\usepackage{etoolbox}

\usepackage{multirow}%
\usepackage{amsthm}%
\usepackage{mathrsfs}%
\usepackage[title]{appendix}%
\usepackage{textcomp}%
\usepackage{manyfoot}%
\usepackage{booktabs}%
\usepackage{algorithmicx}%
\usepackage{listings}%

\usepackage{subfigure}
\usepackage{graphicx}
\usepackage{epstopdf}
\ifpdf
  \DeclareGraphicsExtensions{.eps,.pdf,.png,.jpg}
\else
  \DeclareGraphicsExtensions{.eps}
\fi

\usepackage{todonotes}

\usepackage[mathlines]{lineno}

\DeclareMathOperator*{\argmin}{arg\,min}   

\newcommand{\R}{\mathbb{R}}
\newcommand{\C}{\mathbb{C}}
\newcommand{\N}{\mathbb{N}}

\renewcommand{\H}{\mathcal{H}}
\newcommand{\x}{\boldsymbol{x}}
\newcommand{\y}{\boldsymbol{y}}
\newcommand{\X}{\mathcal{X}}
\renewcommand{\u}{\boldsymbol{u}}
\newcommand{\w}{\boldsymbol{w}}

\newcommand{\D}{\mathbb{D}}
\newcommand{\bs}{\boldsymbol}
\newcommand{\mc}{\mathcal}

\newcommand{\im}{\mathrm{i}}
\newcommand{\e}{\mathrm{e}}
\renewcommand{\d}{\,\mathrm{d}} 
\newcommand{\norm}[1]{\left\lVert \smash{#1} \right\rVert}

\DeclareMathOperator\supp{supp}

\newtheorem{theorem}{Theorem}[section]
\newtheorem{lemma}[theorem]{Lemma}
\newtheorem{corollary}[theorem]{Corollary}
\newtheorem{proposition}[theorem]{Proposition}

\theoremstyle{definition}
\newtheorem{definition}[theorem]{Definition}

\numberwithin{equation}{section}

\title{A Weighted Kernel Method for Approximation that Adapts to Learned Multivariable Structure}
\author{John E. Darges\orcidlink{0009-0007-9059-8921}}
\thanks{John E. Darges: Department of Mathematics, Emory University. Email: \texttt{jdarges@emory.edu}}

\author{Laura Weidensager\orcidlink{0000-0001-7988-1485}}
\thanks{Laura Weidensager: Department of Mathematics, Simon Fraser University. Email: \texttt{laura\_weidensager@sfu.ca}}
\date{\today}

\begin{document}

\begin{abstract}
  Approximating the input-output behavior of a multivariable black-box function from limited data is challenging when blind to the importance of its inputs and their interactions. We introduce total sensitivity kernels (TSKs), a method based on families of weighted ANOVA kernels that learn and adapt to this multivariable structure. TSKs parameterize the weights on each multivariable component of the target function by factors for each input. We propose learning these factors directly from function evaluations by selecting the reproducing kernel Hilbert space (RKHS) in which the target function has minimum norm. 
Under suitable conditions, we show that this norm-minimization problem admits a unique solution, and we establish consistency of a finite-data formulation based on minimum-norm interpolation. The learned TSK factors characterize the participation of individual inputs across interactions and main effects, providing a kernel-dependent notion of input sensitivity related to total Sobol' indices. Numerical experiments demonstrate that adapting the kernel to learned multivariable structure can substantially improve approximation accuracy over a standard product kernel and ANOVA kernel.

\vspace{2ex}\noindent

\end{abstract}
\maketitle



\section{Introduction}

Consider approximating a black box function $f:\R^d\to\C$ from a limited number of function evaluations. When $f$ is high-dimensional,  accurate approximation requires many function evaluations if the approximation method does not adapt to multivariable structure. For many functions in applications, some inputs are more influential than others and only certain groups of inputs interact strongly. Effective methods must exploit, within the budget, multivariable structure not known a priori.

We describe multivariable structure through the functional analysis of variance (ANOVA) decomposition~\cite{Hoeffding48,sobol1969multidimensional,Antoniadis1984,Kuo10}. With respect to a product measure on the input domain, $f$ can be decomposed into mutually orthogonal components
\begin{equation}\label{eq:ANOVA_decomp}
f(\x) = \sum_{\u\subseteq \{1,\ldots,d\}} f_{\u}( \x_{\u}).
\end{equation}
Each term $f_{\u}(\x_{\u})$ represents the contribution of input groupings indexed by $\u\subseteq \{1,\ldots, d\}$. With $2^d$ subsets, the number of potentially relevant components grows exponentially with the input dimension.

Many functions in applications, however, have a low effective dimension~\cite{caflisch1997valuation,liu2006estimating}. Functions with low effective dimension have most of their variation explained by a few ANOVA components. 
High-dimensional approximation is tractable if we can concentrate on the most relevant portion of the ANOVA components rather than treating them all equally.

As high-dimensional functions may possess exploitable low-dimensional structure, researchers have developed methods for structured approximation. High-dimensional model representations (HDMR) construct approximations from low-dimensional component functions by assuming that high-order interactions make small contributions~\cite{Rabitz99}. Another approach estimates the effective dimension or variable importance to adapt anchored-ANOVA approximations~\cite{Choi12adaptive}. Related ANOVA-based methods include Fourier approximation~\cite{potts2022learning}, random feature expansions~\cite{elm,potts24}, and tensor train approximation~\cite{bbtensor}. These methods screen for or require tuning to important variables and interaction terms.
The statistical learning community has used sparse learning and variable selection to identify the most relevant variables and interaction terms~\cite{bach2008exploring,fang2016flexible}. Broadly, these methods impose, infer, or exploit a reduced structural representation of the target function.

Another perspective is to encode multivariable structure into the approximation space. One does this by constructing a norm that places weights on variables and interaction terms. Choice of weights imposes a geometry that aligns to a particular multivariable structure. Weighted function spaces appear in the analysis of high-dimensional integration to account for unequal importance of inputs and interactions~\cite{Sloan98}. Weighted ANOVA spaces use collections of weights $\{\gamma_{\bs u}\}_{\bs u\subseteq[d]}$ to rescale contributions of a function's multivariable components. Weights determine how strongly each component is penalized by the norm, so that the same target function has a very different norm depending on how well the weights align with its multivariable structure~\cite{EffDim,Kritzer21qmc}. Appropriate choice of weights explains tractability of quasi-Monte Carlo (QMC) integration in high dimensions~\cite{Sloan98,DickKuoSloan13,EffDim}. Structured, parameterized weights, like product weights and product-order-dependent (POD) weights, assign weight values based on a shared set of weights on input variables and interaction order~\cite{KuoSloan12qmc,DickKuoSloan13}. Parameterized weights avoid having to assign $2^d$ individual values for each weight. There is no all-encompassing method for selecting appropriate weights for QMC integration~\cite{larcher2003tractability,dick2012random}. 

In our setting, the goal is to select, using the limited function evaluations at hand, a  weighted space suitable for high-dimensional approximation.  We frame our approach in weighted reproducing kernel Hilbert spaces (RKHS), where the weights that determine the norm also determine the expression of the associated kernel function. Specifically, ANOVA kernels are our foundation. ANOVA kernels decompose into components associated with subsets of input variables and mirror the ANOVA decomposition~\eqref{eq:ANOVA_decomp}
\begin{equation}\label{equ:anova_kernel}
    \kappa_{\text{ANOVA}}(\bs x,\bs y)=\prod_{k=1}^d\Big(1 + \kappa_k(x_k,y_k)\Big)=\sum_{\bs u\subseteq \{1,\ldots,d\}}\kappa_{\bs u}(\bs x_{\bs u},\bs y_{\bs u}),
\end{equation}
see also~\cite{DuGiRoCa13}.
Standard ANOVA kernels have unweighted components, but current statistical methods learn a sparse set of weights for the components~\cite{bach2008exploring,skimfa}.
Weighted RKHSs have seen use in deterministic kernel-based approximation for partial differential equations (PDE), with a set of structured weights chosen based on analytical properties of the PDE~\cite{kaarnioja2021kernel}. We aim to learn a factors for a weighted kernel that adapts the RKHS to the multivariable structure of the target function.

We introduce total sensitivity kernels (TSKs), a family of weighted ANOVA kernels whose component weights are determined by $d$ parameters, called TSK factors, that represent the importance of each input.
\begin{equation*}
\kappa_{\bs\Sigma}(\bs x,\bs y)
=
\prod_{k=1}^d
\Big(1-\Sigma_k+\Sigma_k\kappa_k(x_k,y_k)\Big)
=
\sum_{\bs u\subseteq[d]}
\gamma_{\bs u}
\kappa_{\bs u}(\bs x_{\bs u},\bs y_{\bs u}),
\qquad
\gamma_{\bs u}
=
\prod_{k\in\bs u}\Sigma_k
\prod_{l\notin\bs u}(1-\Sigma_l).
\end{equation*}
Like product weights in QMC, TSK factors parameterize the $2^d$ ANOVA component weights. TSK factors define a parameterized family of RKHSs from which to choose the approximation space.. We choose TSK factors that determine the RKHS where the target function has minimal norm and so has geometry that aligns with the function's multivariable structure.
Under suitable conditions, we show that determining these weights is a strongly convex optimization problem and therefore admits a unique solution. 
The TSK method provides a principle for learning these parameters from available function evaluations. 
This population problem motivates a practical finite-data procedure where the norm of the unknown target is replaced by the norm of its minimum-norm interpolant. The finite-data objective is related to maximum likelihood estimation objectives that appear in Gaussian process methods such as automatic relevance determination (ARD)~\cite{williams2006gaussian}. We establish theoretical consistency of the resulting finite-data objective.  The method considered here is developed for deterministic approximation, so that function evaluations are assumed to be noise-free, and the input variables are assumed to be independent. 

\subsubsection*{Structure of this work}
We start with Section~\ref{sec:background}, where we introduce known results about kernel spaces and the ANOVA decomposition. In Section~\ref{sec:main} we define our proposed kernel, a weighted kernel, which adapts to the multivariate structure of the target function. We state properties of this kernel, especially in Theorem~\ref{thm:convexity} where we show strong convexity of our main objective which determines TSK factor optimality. Section~\ref{sec:alg} describes our main optimization algorithm in more detail and the numerical challenges that come with the algorithm. 
Finally, in Section~\ref{sec:numerics} we give numerical results, which underpin the benefits of our proposed algorithm.

\section{Background}\label{sec:background}
We start with the introduction of some notation. We use boldface for vectors and matrices. Denote the domains $\D,\Omega\subseteq\R$. Throughout, let $\bs x\in\D^d$ and $\bs\omega\in\Omega^d$ be $d$-dimensional vectors. We denote $[d] \coloneqq \{1,\ldots,d\}$ and the index $\bs u\subseteq [d]$ indicates a subset of the inputs. 
For a vector $\bs x$, the truncated vector $\bs x_{\bs u} \coloneqq (x_k)_{k\in \bs u}$ features only those inputs associated to $\bs u$. If $\bs x^i$ is the $i$th point in a data set, we denote the $k$th entry as $x_{k}^i$.

\subsection{Reproducing Kernel Hilbert Spaces}
 Consider a feature map $\Phi:\D^d\times\Omega^d\to\C$ along with a probability measure $\tau$ on $\Omega^d$. We take the feature map to be the Fourier map $\Phi(\bs x,\bs\omega)=\e^{\im\langle\bs \omega,\bs x\rangle}$ which, due to Bochner's theorem~\cite{BochnersTheorem,Ker_method}, induces a reproducing kernel 
$\kappa\colon\D^d\times\D^d\to\C$ via
\begin{equation}\label{eq:bochner}
\kappa(\x,\y)=\int_{\Omega^d}\e^{\im \langle \bs\omega, \x-\y\rangle}\,\d\tau(\bs\omega).
\end{equation}
For comprehensive theory and properties of reproducing kernels, see~\cite{AronszajnReview,CuckerSmale02,BachA}. We recall some basic properties and theories.
There is a unique Hilbert space $\mathcal{H}$ associated with the kernel $\kappa\colon \D^d\times \D^d \rightarrow \C$, known as a reproducing kernel Hilbert space (RKHS) $\H$ of functions $f\colon \D^d\rightarrow \C$, with the property that $\kappa(\x, \cdot)\in \mc H$, for all $\bs x\in\D^d$, and 
\begin{equation*}
f(\x) = \langle \kappa(\x,\cdot), f\rangle_{\mc H},
\end{equation*}
for all $\bs x\in\D^d$. The reproducing kernel $\kappa$ has the key property that, for distinct points $\bs x^1,\dots,\bs x^M\in\D^d$, the $M\times M$ matrix $\bs K$, where $\bs K_{i,j}=\kappa(\bs x^i,\bs x^j)$, is symmetric positive definite.
Consider the operator $T\colon L_2(\Omega^d,\tau) \rightarrow L_2(\D^d)$,
\begin{equation*}
(T\beta)(\x)=\int_{\Omega^d} \beta(\bs\omega)\e^{\im \langle \bs\omega, \x\rangle}\d \tau(\bs\omega)  ,\quad \beta\in L_2(\Omega^d,\tau).
\end{equation*}
The image of $T$ is the RKHS $\mathcal{H}$. The norm of $f$ in the RKHS $\mc H$ is then
\begin{equation*}
\norm{f}^2_{\mc H} = \min_{T\beta = f} \norm{\beta}^2_{L_2(\Omega^d,\tau)}.
\end{equation*}
The function $\beta$ is called the spectral factor of $f$.
In this work, we consider  kernels with a tensor product structure. These are multivariable kernels that are the Fourier transform of a probability measure which corresponds to independently distributed random variables. The probability measure takes the form $\tau=\tau_1\otimes\dots\otimes\tau_d,$ where each $\tau_k$ is a one-dimensional probability measure. The kernel can be written
\begin{equation}\label{eq:product_form}
    \kappa(\boldsymbol{x},\boldsymbol{y}) = \int_{\Omega^d}\e^{\im\langle\bs\omega,\bs x-\bs y\rangle}\d\tau(\boldsymbol{\omega})=\prod_{k=1}^d\int_{\Omega}\e^{\im\omega_k(x_k-y_k)}\,\d\tau_k(\omega_k).
\end{equation}
We often take $\tau_k$ to be continuous distributions.

Kernel learning is a subset of machine learning. It involves searching an RKHS for a function that closely approximates a target function or a set of data. An RKHS is the closure of the span of $\{\kappa(\bs x,\cdot);\,\bs x\in\X\}$.  So kernel learning means approximating by an expansion $\sum_{i=1}^M \beta_i\kappa(\bs x_i,\cdot)$.


\subsection{Kernel ridge regression and generalization bounds}
Kernel ridge regression (KRR) is a cornerstone algorithm for non-parametric function estimation. Consider a dataset $\{(\bs x^i, f^i)\}_{i=1}^M$. Here, each $\bs x_i\sim\mu$ is drawn identically and independently, while $f(\bs x^i)=f^i$.
KRR finds an estimator $\hat{f}_\lambda$ by minimizing the regularized empirical risk over the RKHS $\mc H$ \cite{ScSm01, Stein08},
\begin{equation*}
    \hat{f}_\lambda = \argmin_{f \in \mc H} \frac{1}{M} \sum_{i=1}^M \big(f(\bs x^i) - f^i\big)^2 + \lambda \norm{f}_{\mc H}^2,
\end{equation*}
where $\lambda > 0$ is the regularization parameter controlling the trade-off between data fit and generalization. \par

The theoretical performance of the estimator $\hat{f}_\lambda$ is typically evaluated via its expected generalization error in $L_2(\D^d,\mu)$. Assuming the true target function $f$ is contained within the RKHS $\mc H$, standard results from statistical learning theory and integral operator techniques \cite{CaDe07, SmZh07} provide a high-probability upper bound on the squared $L_2(\D^d,\mu)$ error,
\begin{equation}\label{eq:krr_bound}
    \norm{\hat{f}_\lambda - f}_{L_2(\mu)}^2 \leq \mathcal{O}\left( \frac{K^2 \log(1/\delta)}{\lambda M} \right) + \lambda \norm{f}_{\mc H}^2,
\end{equation}
where $K= \sup_{\bs x\in \D^d} \sqrt{\kappa(\bs x, \bs x)}$ bounds the kernel and $1-\delta$ is the confidence level. \par

The bound in~\eqref{eq:krr_bound} decomposes the error into two competing terms. The first term represents the estimation error, which vanishes as the sample size $M \to \infty$. When restricting to shift-invariant kernels, this term is independent of the choice of kernel since $K=1$ for all shift-invariant kernels. The second term represents the approximation bias introduced by the Tikhonov regularization. Balancing these two terms optimally by setting $\lambda \sim 1/\sqrt{M}$ yields an overall learning rate of $\mathcal{O}\big(\norm{f}_{\mc H} M^{-1/2}\big)$. \par

Equation~\eqref{eq:krr_bound} reveals a fundamental principle for methods using shift-invariant kernels: the tightness of the generalization bound is strictly governed by the RKHS norm of the true function, $\norm{f}_{\mc H}$. Selecting a kernel that actively minimize the objective $ \norm{f}_{\mc H}$ explicitly minimizes the leading constant of the approximation bias in~\eqref{eq:krr_bound}. 

\section{Theory of total sensitivity kernels}\label{sec:main}
In this section we propose the procedure to learn a kernel $\kappa$ that adapts to the multivariable structure of the target function. 
The kernels take a form similar to ANOVA kernels~\eqref{equ:anova_kernel}. We modify the product kernel~\eqref{eq:product_form} by introducing tunable parameters that control input importance.
\begin{definition}\label{def:tsk}
    For each  $k\in [d]$, let $\kappa_k$ be a univariate kernel. A total sensitivity kernel (TSK) is defined by
    \begin{equation}\label{eq:tsk}
        \kappa_{\bs\Sigma}(\bs x,\bs y) = \prod_{k=1}^d\big(1-\Sigma_k + \Sigma_k\kappa_k(x_k,y_k)\big).
    \end{equation}
    There, we denote the TSK factors $\bs\Sigma=\left(
   \Sigma_1,  \dots , \Sigma_d 
    \right)$, with $\Sigma_k\in(0,1).$ Let $\H_{\bs \Sigma}$ be the RKHS associated to $\kappa_{\bs \Sigma}$, with norm $\|\cdot\|_{\bs\Sigma}$.
\end{definition}
The TSK~\eqref{eq:tsk} takes an ANOVA form when we multiply out,
\begin{equation*}
    \kappa_{\bs\Sigma}(\bs x,\bs y) = \sum_{\bs u\subseteq[d]}\gamma_{\bs u}\kappa_{\bs u}(\bs x_{\bs u},\bs y_{\bs u}),\quad \kappa_{\bs u}=\prod_{k\in\bs u}\kappa_k,\, \gamma_{\u} = \prod_{k\in \bs u}\Sigma_k\prod_{l\notin u}(1-\Sigma_l).
\end{equation*}
Therefore, an input's TSK factor being larger means that interaction terms corresponding to that input are weighted more.
In this work, we let the univariate kernels be shift-invariant kernels, specifically those that come from characteristic functions,
\begin{equation*}
    \kappa_k(x_k,y_k) = \phi_k(x_k-y_k)=\int_{\Omega_k} \e^{\im\omega_k(x_k-y_k)}\,\d\tau_k(\omega_k).
\end{equation*}
Here, $\phi_k$ is the characteristic function of a distribution $\tau_k$. Then, $\kappa_{\bs u}(\bs x_{\bs u},\bs y_{\bs u}) = \prod_{k\in\bs u}\phi_k(x_k-y_k)$.

\subsection{RKHS properties}
We discuss the properties of the TSKs and characterize the corresponding RKHS. As a corollary of \cite[Part I.6]{AronszajnReview}, we have the following lemma for the norm in a sum of RKHSs.
\begin{lemma}\label{lem:norm_sum}
    Let $\kappa_{{\bs  \Sigma}}$ be a TSK. For $f\in\mc{H}_{\bs\Sigma}$, with corresponding RKHS norm $\|\cdot\|_{\bs  \Sigma}$, then
    \begin{equation}\label{equ:ker_sum}
    \|f\|_{{{\bs  \Sigma}}}^2=\min_{f=\sum_{\bs u} f_{\bs u}} \sum_{\bs u\subseteq [d]}\frac{1}{\gamma_{\bs u}}\|f_{\bs u}\|_{\mc{H}_{\bs u}}^2, \quad f_{\bs u}\in\mc H_{\bs u}.
    \end{equation}
    Each RKHS $\mc H_{\bs u}$ corresponds to the kernel $\kappa_{{\bs u}}=\prod_{k\in\bs u}\kappa_k$ with attached norm $\|\cdot\|_{\mc H_{\bs u}}$.
\end{lemma}
\begin{proof}
    Each $\gamma_{\bs u}\kappa_{\bs u}$ is a kernel with an associated RKHS, which we denote $\mc H_{\gamma,\bs u}$ with corresponding norm $\|\cdot\|_{\gamma,\bs u}$. According to~\cite[Part I.6]{AronszajnReview} about the sum of reproducing kernels, the norm of $f\in\mc H_{\bs \Sigma}$ is 
    \begin{equation*}
    \|f\|_{{\bs \Sigma}}^2=\min_{f=\sum_{\bs u} f_{\bs u}} \sum_{\bs u\subseteq [d]}\|f_{\bs u}\|^2_{\mc H_{\gamma,\bs u}}, \quad f_{\bs u}\in\mc H_{\gamma,\bs u}.
    \end{equation*}
    Yet, note that the class of functions comprising $\mc H_{\bs u}$ is the same as that of $\mc H_{\gamma,\bs u}$. If $f\in\mc H_{\gamma,\bs u}$, then we can express it as
    \begin{equation*}
    f = \sum_{i=1}^M\alpha_i(\gamma_{\bs u}\kappa_{\bs u}(\bs x^i,\cdot)).
    \end{equation*}
    It is clear that $f=\sum_{i=1}^M\beta_i\kappa_{\bs u}(\bs x^i,\cdot)$, with $\beta_i=\gamma_{\bs u}\alpha_i$ so that $f\in\mc H_{\bs u}$. The norm of $f$ in $\mc H_{\gamma,\bs u}$ is 
    \begin{equation*}
    \|f\|^2_{\mc H_{\gamma,\bs u}}=\sum_{i,\ell=1}^M\alpha_i\alpha_\ell(\gamma_{\bs u}\kappa_{\bs u}(\x^i,\x^{\ell})) = \frac{1}{\gamma_{\bs u}} \sum_{i,\ell=1}^M\beta_i\beta_\ell\kappa_{\bs u}(\x^i,\x^\ell) = \frac{1}{\gamma_{\bs u}}\|f\|^2_{\mc H_{\bs u}}.
    \end{equation*}
    Summing all terms gives
    \begin{equation*}
        \min_{f=\sum_{\bs u} f_{\bs u}} \sum_{\bs u\subseteq [d]}\|f_{\bs u}\|^2_{\mc H_{\gamma,\bs u}}=\min_{f=\sum_{\bs u} f_{\bs u}} \sum_{\bs u\subseteq [d]}\frac{1}{\gamma_{\bs u}}\|f_{\bs u}\|^2_{\mc H_{\bs u}}.
    \end{equation*}
    This finishes the proof.
\end{proof}
The lemma shows an immediate connection to the ANOVA decomposition. Due to the orthogonality of the terms in the ANOVA decomposition~\eqref{eq:ANOVA_decomp}, there we have $\|f\|^2=\sum_{\bs u\subseteq[d]}\|f_{\bs u}\|^2$. So, the RKHS norm $\| f\|_{\bs\Sigma}$ corresponding to the TSK $\kappa_{\bs \Sigma}$ resembles a weighted sum of ANOVA energies in the RKHS. 

In the discussion of kernels of the form~\eqref{eq:bochner}, we see every probability measure induces a kernel. TSKs are induced by a probability measure. This gives an alternative formulation of the kernel,
\begin{equation*}
    \kappa_{\bs \Sigma}(\bs x,\bs y) = \int_{\Omega^d} \e^{\im\langle\bs\omega,\bs x-\bs y\rangle}\,\d\tau_{\bs\Sigma}(\bs\omega),\quad  \tau_{{\bs  \Sigma}} = \sum_{\bs u\subseteq [d]}\big(\bigotimes_{k\in \bs u}\Sigma_k\tau_k\bigotimes_{k\notin \bs u}(1-\Sigma_k)\delta_0                 \big), 
\end{equation*} 
where $\delta_0$ is the Dirac measure. This shows that a TSK is induced by a spike-and-slab distribution. Hence a TSK is a shift-invariant kernel as long as the 1D kernels $\kappa_k$ are shift-invariant. We can interpret the TSK factor $\Sigma_k$ as the probability of ``activation" for the input $x_k$. 

The TSK factors parameterize a family of probability measures, kernel functions, and RKHSs. Choosing TSK factors means choosing which RKHS to approximate $f$ in. What is the relationship between the RKHSs corresponding to two different sets of TSK factors?
In fact, the RKHSs describe the same sets of functions, no matter how the TSK factors $\bs \Sigma$ are chosen. Changing TSK factors only changes the RKHS inner product. The following proposition states this in more detail.
\begin{proposition}\label{prop:rnderiv}
    Consider a family of TSKs parameterized by $\bs\Sigma$.
     Let ${\bs  \Sigma}\in(0,1)^d$ and $\widetilde{{\bs  \Sigma}}\in[0,1]^d$. Then any $f\in\mathcal{H}_{\widetilde{{\bs  \Sigma}}}$ is also a member of $\mathcal{H}_{{\bs  \Sigma}}$.
\end{proposition}
\begin{proof}
    For each $f\in\mathcal{H}_{\widetilde{{\bs  \Sigma}}}$, there exists a spectral factor $\widetilde{\beta}$ in $L_2(\tau_{\widetilde{{\bs  \Sigma}}})$ such that
    \begin{equation*}
f(\bs x) = \int \widetilde{\beta}(\bs\omega)\e^{\im \bs\omega^\top\bs x}\d\tau_{\widetilde{{\bs  \Sigma}}}(\bs\omega) 
= \sum_{\bs u\subseteq[d]}\widetilde{\gamma}_{\bs u}\int \widetilde{\beta}_{\bs u}^0\e^{\im\bs\omega_{\bs u}^\top\bs x_{\bs u}}\d\tilde \tau_{\bs u}(\bs\omega_{\bs u}),
\end{equation*}
where $\tilde{\tau}_{\bs u} \coloneqq \bigotimes_{k\in\bs u}\tau_{\tilde \Sigma_k}$ and $\widetilde{\gamma}_{\bs u}=\prod_{k\in\bs u}\tilde\Sigma_k\prod_{l\notin\bs u}(1-\tilde\Sigma_l)$.
Now, define the function
\begin{equation*}
\frac{\d\tau_{\widetilde{{\bs  \Sigma}}}}{\d\tau_{{\bs  \Sigma}}}(\bs\omega)= \left\{\begin{array}{ccc}
\frac{\widetilde{\gamma}_{\emptyset}}{\gamma_{\emptyset}},\quad \bs\omega\in\Omega_{\emptyset} \\
\frac{\widetilde{\gamma}_{\{1\}}}{\gamma_{\{1\}}},\quad \bs\omega\in\Omega_{\{1\}} \\
\vdots \\
\frac{\widetilde{\gamma}_{\bs u}}{\gamma_{\bs u}},\quad \bs\omega\in\Omega_{\bs u} \\
\vdots \\
\frac{\widetilde{\gamma}_{[d]}}{\gamma_{[d]}},\quad \bs\omega\in\Omega_{[d]} \end{array}  \right..
\end{equation*}
This function is piecewise constant on the $2^d$ different sets $\Omega_{\bs u}\subset\R^d$, with $\Omega_{\bs u}=\{\w\in\R^d\colon \supp \w = \u\}$. For example, $\Omega_\emptyset = \{0\}$, $\Omega_{\{1\}}=\{\bs\omega\in\R^d\colon \omega_k=0,\ k\neq 1\}$, and $\Omega_{[d]}=\{\bs\omega\in\R^d\colon\omega_k\neq 0\}$.  The function $\frac{d\tau_{\widetilde{{\bs  \Sigma}}}}{d\tau_{{\bs  \Sigma}}}(\bs\omega)$ is the Radon-Nikodym derivative. This exists when $\tau_{\widetilde{{\bs  \Sigma}}}$ is absolutely continuous with respect to $\tau_{\bs \Sigma}$, which is the case whenever ${\bs  \Sigma}\in(0,1)^d$. With these we can write
    \begin{align*}
    f(\bs x)  =& \sum_{\bs u\subseteq[d]}\widetilde{\gamma}_{\bs u}\int_{\Omega^{|\bs u|}} \widetilde{\beta}_{\bs u}^0\e^{\im\bs\omega_{\bs u}^\top\bs x_{\bs u}}\d\tilde\tau_{\bs u}(\bs\omega_{\bs u}) 
    = \sum_{\bs u\subseteq[d]}\gamma_{\bs u}\int_{\Omega^{|\bs u|}} \frac{\widetilde{\gamma}_{\bs u}}{\gamma_{\bs u}}\widetilde{\beta}_{\bs u}^0\e^{\im\bs\omega_{\bs u}^\top\bs x_{\bs u}}\d\tilde \tau_{\bs u}(\bs\omega_{\bs u}) \\
    = & \int \widetilde{\beta}(\bs\omega)\e^{\im\bs\omega^\top\bs\x}\frac{d\tau_{\widetilde{{\bs  \Sigma}}}}{d\tau_{{\bs  \Sigma}}}(\bs\omega)\d\tau_{\bs u}(\bs\omega_{\bs u}).
    \end{align*}
    This shows that $f$ can be expressed through the new spectral factor $\beta = \widetilde{\beta}\frac{\d\tau_{\widetilde{{\bs  \Sigma}}}}{\d\tau_{{\bs  \Sigma}}}$ in $\mathcal{H}_{{\bs  \Sigma}}$. To ensure $f\in \mathcal{H}_{{\bs  \Sigma}}$, we verify that $\beta \in L_2(\Omega,\tau_{{\bs  \Sigma}})$. Let $m \coloneqq \max_{\bs u\subseteq[d]} \frac{\widetilde{\gamma}_{\bs u}}{\gamma_{\bs u}} < \infty$. Then,
    \begin{equation*}
    \|\beta\|_{L_2(\tau_{{\bs  \Sigma}})}^2 = \int |\widetilde{\beta}|^2 \left(\frac{d\tau_{\widetilde{{\bs  \Sigma}}}}{d\tau_{{\bs  \Sigma}}}\right)^2 d\tau_{{\bs  \Sigma}} \le m \int |\widetilde{\beta}|^2 \frac{d\tau_{\widetilde{{\bs  \Sigma}}}}{d\tau_{{\bs  \Sigma}}} d\tau_{{\bs  \Sigma}} = m\|\widetilde{\beta}\|_{L_2(\tau_{\widetilde{{\bs  \Sigma}}})}^2 < \infty.
    \end{equation*}
    Since $\beta$ has a finite RKHS norm under $\tau_{{\bs  \Sigma}}$, we conclude $f\in \mathcal{H}_{{\bs  \Sigma}}$.
\end{proof}
Choosing which RKHS we approximate $f$ in, across all TSK factors in $(0,1)^d$, is then a well-posed question. It is impossible for the target to be in one RKHS but not another.

\subsection{Optimal TSK factors}
We now must answer what makes TSK factors optimal. We want the structure of $\kappa_{\bs\Sigma}$ to match the multivariable structure of $f$. Recall that the corresponding RKHS norm in~\cref{lem:norm_sum} weighs contributions by different groupings of inputs differently. Since each RKHS includes the same set of functions, different weights align the geometry of the space with a specific multivariable structure. A small value of the weight $\gamma_{\bs u}$ heavily penalizes a large $\bs u$-component of $f$ by inflating the RKHS norm. Conversely, a large value of $\gamma_{\bs u}$ does not penalize a large $\bs u$-component. So, the multivariable structure of $f$ aligns most closely with the RKHS where the norm is smallest. Identifying the RKHS where $f$ has the minimal norm also lowers one of the terms in the upper bound on the KRR error in~\eqref{eq:krr_bound}.

We therefore optimize the TSK factors ${\bs  \Sigma} = (\Sigma_1,\ldots,\Sigma_d)$ in~\eqref{eq:tsk} by minimizing the RKHS norm,
\begin{equation*}
g(\bs\Sigma) \coloneqq \norm{f}^2_{{\bs  \Sigma}}.
\end{equation*}
The following result shows that for a fixed function $f$ the norm $g({\bs  \Sigma})$ as a function depending on the parameters $\bs \Sigma = (\Sigma_1,\ldots,\Sigma_d)$ is not only convex, but also strongly convex, which means that the spectrum of the Hessian is bounded from below by a constant bigger than zero. Strongly convex functions have the benefit that they have a unique global minimum, if it exists.

\begin{theorem}\label{thm:convexity}
Let $ 0\neq f\in \mc H_{\bs\Sigma}$ for some, and hence every,
$\bs\Sigma\in(0,1)^d$. Assume that each kernel $\kappa_k$ is a strictly positive definite kernel satisfying
$\kappa_k(x_k,x_k)=1.$
Then the function
\[
    g(\bs\Sigma)=\|f\|_{\bs\Sigma}^2
\]
is strongly convex on $(0,1)^d$.
\end{theorem}

\begin{proof}
We first construct a decomposition of each one-dimensional kernel $\kappa_k$ into a
constant part and a part whose RKHS contains no nonzero constant
functions. Define  $r_k\coloneqq \|1\|_{\mc H_k}^2$. For every $k\in [d]$, there are two cases and we define the constants
\begin{equation}\label{equ:Ck}
C_k\coloneqq \begin{cases}
    1 &\text{ if } 1\notin\mc H_k,\\
     1-\frac{1}{r_k} &\text{ if } 1\in\mc H_k.
\end{cases}\end{equation}
Then we define the RKHS, which contains no constant by
\[\widetilde{\kappa}_k(x_k,y_k)\coloneqq \begin{cases}
    \kappa_k(x_k,y_k) &\text{ if } 1\notin\mc H_k,\\
    \kappa_k(x_k,y_k)-\frac{1}{r_k} &\text{ if } 1\in\mc H_k.
\end{cases}\]
To see that $\widetilde{\kappa}_k$ is positive definite in the second case,
let $\{x_k^1,\ldots,x_k^M\}\subset\D$ and $a_1,\ldots,a_M\in\R$.
By the reproducing property,
\begin{align*}
\sum_{i,j=1}^M a_i a_j
\widetilde{\kappa}_k(x_k^i,x_k^j)
&=
\Big\|\sum_{i=1}^M a_i\kappa_k(x_k^i,\cdot)\Big\|_{\mc H_k}^2
-\frac{1}{r_k}
\Big(\sum_{i=1}^M a_i\Big)^2\\
&=
\Big\|\sum_{i=1}^M a_i\kappa_k(x_k^i,\cdot)\Big\|_{\mc H_k}^2
-\frac{
\Big|
\Big\langle
1,\sum_{i=1}^M a_i\kappa_k(x_k^i,\cdot)
\Big\rangle_{\mc H_k}
\Big|^2
}{
\|1\|_{\mc H_k}^2
}
\geq 0,
\end{align*}
where the final inequality follows from the Cauchy--Schwarz inequality.
Now, if $1\in\mc H_k$, $\mc H_k=\mc H^1_k\oplus\mc H_{\widetilde{\kappa}_k}$
where $\mc H_k^1=\operatorname{span}\{1\}$ with norm inherited from $\mc H_k$. The reproducing kernel of this
one-dimensional subspace is the constant kernel $1/r_k$.

We next show that 
    $0< C_k\leq 1.$
The result is immediate if $1\notin\mc H_k$. Suppose instead that
$1\in\mc H_k$. Since $\kappa_k(x_k,x_k)=1$, the reproducing property
gives
\[
    |h(x_k)|
    \leq
    \|h\|_{\mc H_k}
    \|\kappa_k(x_k,\cdot)\|_{\mc H_k}
    =
    \|h\|_{\mc H_k}
\]
for every $h\in\mc H_k$. Taking $h=1$ gives $1\leq\|1\|_{\mc H_k}.$
Equality would imply equality in the Cauchy--Schwarz inequality for
$1$ and $\kappa_k(x_k,\cdot)$ for every $x_k\in\D$. This would mean
    $\kappa_k(x_k,\cdot)=1$
for every $x_k$, and hence $\kappa_k\equiv1$, contradicting strict
positive definiteness on a domain containing at least two distinct
points. Then, 
    $r_k=\|1\|_{\mc H_k}^2>1$
and 
\[
    0<C_k=1-\frac{1}{r_k}<1.
\]
In either case,
\[
    0< C_k\leq 1.
\]

The $k$th factor of the TSK can be written as
\begin{align*}
    1-\Sigma_k+\Sigma_k\kappa_k(x_k,y_k) = 1+\Sigma_k\big(\kappa_k(x_k,y_k)-1\big)
    &=
1+\Sigma_k\big(\widetilde{\kappa}_k(x_k,y_k)-C_k\big),
\end{align*}
since in both cases $\widetilde{\kappa}_k(x_k,y_k) - C_k = \kappa_k(x_k,y_k)-1$. 
Therefore,
\begin{equation}\label{eq:tsk_orthogonal_decomposition}
    \kappa_{\bs\Sigma}(\x,\y)
    =
    \prod_{k=1}^d
    \big(
        1-C_k\Sigma_k
        +\Sigma_k\widetilde{\kappa}_k(x_k,y_k)
    \big).
\end{equation}

\textbf{Decomposition of the RKHS norm:}\\
Expanding~\eqref{eq:tsk_orthogonal_decomposition} gives
\[
    \kappa_{\bs\Sigma}(\x,\y)
    =
    \sum_{\bs u\subseteq[d]}
    \theta_{\bs u}(\bs\Sigma)
    \widetilde{\kappa}_{\bs u}(\x_{\bs u},\y_{\bs u}),\quad 
    \theta_{\bs u}(\bs\Sigma)
    \coloneqq
    \prod_{k\in\bs u}\Sigma_k
    \prod_{\ell\notin\bs u}(1-C_\ell\Sigma_\ell),
\]
where
\[
    \widetilde{\kappa}_{\bs u}
    \coloneqq
    \prod_{k\in\bs u}\widetilde{\kappa}_k,
    \qquad
    \widetilde{\kappa}_{\emptyset}=1.
\]
Since $0< C_k\leq 1$ and $0<\Sigma_k<1$, we have
    $1-C_k\Sigma_k>0.$
So, $\theta_{\bs u}(\bs\Sigma)>0$ for every $\bs u\subseteq[d]$.

By construction, the constant space and
$\mc H_{\widetilde{\kappa}_k}$ have trivial intersection for every $k\in [d]$.
Then, the tensor-product spaces corresponding to distinct subsets
$\bs u$ decompose $\mc H_{\bs\Sigma}$ by direct sum,
\begin{equation}\label{equ:tensor}
\mc H_{\bs\Sigma} = \bigoplus_{\bs u\subseteq[d]}\Big(\bigotimes_{k\in\bs u}\Sigma_k\mc H_{\widetilde{\kappa}_k}\otimes\bigotimes_{\ell\notin\bs u} (1-C_k\Sigma_k)\mc H_k^1\big) .
\end{equation}
Every $f\in\mc H_{\bs\Sigma}$ therefore
has a unique decomposition
\[
    f=\sum_{\bs u\subseteq[d]}f_{\bs u},
    \qquad
    f_{\bs u}\in
    \mc H_{\widetilde{\kappa}_{\bs u}},
\]
where the constant factors are the inactive coordinates $[d]\backslash \bs u$. Note that these component spaces, and therefore the functions
$f_{\bs u}$, do not depend on $\bs\Sigma$.

By the scaling property of RKHS norms,
\begin{align*}
    g(\bs\Sigma)
    =
    \|f\|_{\bs\Sigma}^2
    &=
    \sum_{\bs u\subseteq[d]}
    \frac{
        \|f_{\bs u}\|_{\mc H_{\widetilde{\kappa}_{\bs u}}}^2
    }{
        \theta_{\bs u}(\bs\Sigma)
    }
    =
    \sum_{\bs u\subseteq[d]}
    \frac{\beta_{\bs u}}
    {\theta_{\bs u}(\bs\Sigma)}, \quad   \beta_{\bs u}
    \coloneqq
    \|f_{\bs u}\|_{\mc H_{\widetilde{\kappa}_{\bs u}}}^2.
\end{align*}
The coefficients $\beta_{\bs u}\geq0$ depend only on $f$ and the
chosen kernels and not on $\bs\Sigma$. Since $f\neq0$, there exists
at least one $\bs u\subseteq[d]$ such that $\beta_{\bs u}>0$.\\

\textbf{The Hessian matrix:}\\
For $\bs u\subseteq[d]$, define
\[
    \xi_k^{\bs u}
    \coloneqq
    \begin{cases}
      -  \dfrac{1}{\Sigma_k},
        & k\in\bs u,\\[3mm]
        \dfrac{C_k}{1-C_k\Sigma_k},
        & k\notin\bs u.
    \end{cases}
\]
Then $\xi_k^{\bs u}>0$ for $k\in\bs u$ and
$\xi_k^{\bs u}<0$ for $k\notin\bs u$. Direct differentiation gives
\[
    \bs H
    =
    \nabla^2g(\bs\Sigma)
    =
    \sum_{\bs u\subseteq[d]}
    \frac{\beta_{\bs u}}{\theta_{\bs u}}
    \bs H^{\bs u},
\]
with diagonal entries
$H^{\bs u}_{k,k}=
 2(\xi_k^{\bs u})^2$
and, non-diagonal entries for $k\neq j$,
$H^{\bs u}_{k,j}
    =
    \xi_k^{\bs u}\xi_j^{\bs u}.$
Consequently, for any $\bs v\in\R^d$,
\begin{align*}
    \bs v^\top \bs H^{\bs u}\bs v
    &=
    \sum_{k=1}^d
    v_k^2(\xi_k^{\bs u})^2
    +
    \left(
        \sum_{k=1}^d v_k\xi_k^{\bs u}
    \right)^2.
\end{align*}
In particular, $H^{\bs u}$ is positive definite because
$\xi_k^{\bs u}\neq0$ for every $k$.

It remains to show that the smallest eigenvalue of $\bs H$ is uniformly
bounded below on $(0,1)^d$. Choose $\bs u$ such that
$\beta_{\bs u}>0$. 
Define
\[
    \widehat C\coloneqq
    \min_{k\in[d]}C_k>0
\]
If $k\in\bs u$, then
    $|\xi_k^{\bs u}|=\frac{1}{\Sigma_k}>1>\widehat C.$ 
If $k\notin \bs u$, then 
\[
    |\xi_k^{\bs u}| = \frac{C_k}{1-C_k\Sigma_k} >C_k>\widehat C.\]
Moreover,
$
    0<\theta_{\bs u}(\bs\Sigma)\leq1
$
because $0<\Sigma_k<1$ and
$0<1-C_k\Sigma_k\leq1$. Hence, for any $\bs v$ with
$\|\bs v\|=1$,
\begin{align*}
    \bs v^\top H\bs v
    &\geq
    \frac{\beta_{\bs u}}{\theta_{\bs u}}
    \bs v^\top H^{\bs u}\bs v
    \geq
    \frac{\beta_{\bs u}}{\theta_{\bs u}}
    \sum_{k=1}^d
    v_k^2(\xi_k^{\bs u})^2
    \geq
    \beta_{\bs u}\widehat C^2
    \sum_{k=1}^d v_k^2
    =
    \beta_{\bs u}\widehat C^2>0.
\end{align*}
The lower bound is independent of $\bs v$ and $\bs\Sigma$. Therefore,
\[
    \inf_{\bs\Sigma\in(0,1)^d}
    \lambda_{\min}\big(\nabla^2g(\bs\Sigma)\big)
    \geq
    \beta_{\bs u}\widehat C^2>0,
\]
which proves that $g$ is strongly convex on $(0,1)^d$.
\end{proof}
As a consequence, minimizing TSK factors, if they exists, are unique to a function $f$. 
\subsection{Interpretation of the TSK factors \texorpdfstring{${\bs  \Sigma}$}{Sigma}}
While the TSK factors ${\bs  \Sigma}$ naturally act as a measure of variable importance, it is crucial to carefully distinguish them from classical total Sobol' indices~\cite{sobol1969multidimensional}. 
In standard global sensitivity analysis, the total Sobol' index quantifies the exact variance contribution of a variable (including all its interactions) relative to the total variance of the model~\cite{Sobol01}. This variance decomposition relies fundamentally on the ANOVA decomposition, which requires the functions $f_{\bs u}$ in~\eqref{eq:ANOVA_decomp} to be strictly orthogonal in $L_2$ with respect to a given probability measure.
We do not force orthogonality with our kernel construction in~\eqref{eq:tsk}.
As a consequence, the strict zero-mean property is relaxed, and the $L_2$-orthogonality of the ANOVA components is no longer guaranteed. The components may overlap and exhibit non-zero covariance, meaning the total variance of $\hat{f}$ does not split into an independent sum. Thus, $\Sigma_k$ cannot be interpreted as an exact variance-based Sobol' index. 

Nevertheless, the TSK factors ${\bs  \Sigma}$ maintain a rigorous and highly useful interpretation as a \textit{model-intrinsic, RKHS-based sensitivity weight}. Even without $L_2$-orthogonality, the tensor product structure of $\kappa_{\bs  \Sigma}$ ensures that the RKHS $\mc H_{{\bs  \Sigma}}$ splits into a direct sum of sub-Hilbert spaces, see~\eqref{equ:tensor}.
Having a closer look at the norm $\norm{f}_{\bs  \Sigma}^2$, we see that if we optimize for a single index $k\in[d]$ and fix all other parameters $\Sigma_\ell$, then the RKHS norm has the form
$$\norm{ f}_{\bs  \Sigma}^2 = \frac{A_k}{\Sigma_k}+ \frac{B_k}{1-\Sigma_k},$$
where $A_k$ and $B_k$ do not depend on $\Sigma_k$. Minimizing this expression by setting the derivative to zero yields
\[    \Sigma_k = \frac{\sqrt{A_k}}{\sqrt{A_k} + \sqrt{B_k}}.\]
Therefore, the learned parameters ${\bs  \Sigma} = (\Sigma_1, \dots, \Sigma_d)$ serve as a continuous feature selection mechanism and a proxy for total sensitivity. A value of $\Sigma_k \to 1$ indicates that the $k$-th variable carries substantial explanatory energy necessary to fit the data, whereas $\Sigma_k \to 0$ acts as a shrinkage operator, effectively pruning the irrelevant dimension from the model.

Unfortunately, there is no closed form for the optimal parameters ${\bs  \Sigma}$ in~\cref{thm:convexity}. However, there are coupled equations, which any optimal parameters ${\bs  \Sigma}^*=(\Sigma_1^*,\ldots,\Sigma_d^*)$ fulfill.

Let $C_k$ be defined as in~\eqref{equ:Ck}.
Under the assumptions of Theorem~\ref{thm:convexity}, these constants are bounded by $0<C_k\leq 1$. Recall that $g(\bs\Sigma)=\norm{f}_{\bs\Sigma}^2$ with
\[
\norm{f}_{\bs\Sigma}^2
=
\sum_{\bs u\subseteq[d]}
\frac{\beta_{\bs u}}{\theta_{\bs u}(\bs\Sigma)},
\qquad
\theta_{\bs u}(\bs\Sigma)
=
\prod_{k\in\bs u}\Sigma_k
\prod_{k\notin\bs u}(1-C_k\Sigma_k),\quad
\beta_{\bs u}
\coloneqq
\norm{f_{\bs u}}_{\mc H_{\widetilde{\kappa}_{\bs u}}}^2.
\]
Taking the $k$th partial derivative of
gives
\[
\frac{\partial g}{\partial\Sigma_k}
=
-\frac{S_k^{\mathrm{in}}}{\Sigma_k^2}
+
\frac{C_kS_k^{\mathrm{out}}}{(1-C_k\Sigma_k)^2}, 
\]
where
\[
S_k^{\mathrm{in}}
=
\sum_{\bs u:\,k\in\bs u}
\beta_{\bs u}
\prod_{\substack{\ell\in\bs u\\\ell\neq k}}
\frac{1}{\Sigma_\ell^*}
\prod_{m\notin\bs u}
\frac{1}{1-C_m\Sigma_m^*}, \quad 
S_k^{\mathrm{out}}
=
\sum_{\bs u:\,k\notin\bs u}
\beta_{\bs u}
\prod_{\ell\in\bs u}
\frac{1}{\Sigma_\ell^*}
\prod_{\substack{m\notin\bs u\\m\neq k}}
\frac{1}{1-C_m\Sigma_m^*}.
\]
 Setting the preceding expression
equal to zero and solving for $\Sigma^*_k$ gives
\begin{equation}\label{eq:sigma_k*}
\Sigma_k^*
=
\frac{1}
{C_k+\sqrt{\dfrac{C_kS_k^{\mathrm{out}}}
{S_k^{\mathrm{in}}}}}
=
\frac{\sqrt{S_k^{\mathrm{in}}}}
{C_k\sqrt{S_k^{\mathrm{in}}}
+\sqrt{C_kS_k^{\mathrm{out}}}}.
\end{equation}
 
Note that the exact calculation of the optimal parameter $\Sigma^*_k$ is not possible without knowing the function to calculate the constants $\beta_{\bs u}$. Additionally, the previous calculations yield coupled systems of equations, which cannot be solved directly. However, the formula~\eqref{eq:sigma_k*} shows that the optimal TSK factors $\Sigma_k^*$ can be interpreted as some sort of sensitivity parameters. Observe that $S^{\text{in}}_k$ features indices $\bs u$ that include $k$. As such, the expression for $\Sigma^*_k$ has components of the norm of $f$ associated with $k$ divided by terms that represent all of the function's components.\par

\section{Weight learning: algorithms and theory}\label{sec:alg}
Direct calculation of the minimizer of the true norm $g({\bs  \Sigma})=\norm{f}_{\bs  \Sigma}^2$, described in the previous section, is not possible without knowing the function $f$ exactly. For numerical applications, we propose in this section to calculate the optimal TSK factors by using a discrete analogue instead.

\subsection{Finite-data TSK}
While Theorem~\ref{thm:convexity} establishes the existence of a unique optimal parameter vector ${\bs  \Sigma}^*$ that minimizes the theoretical RKHS norm $g({\bs  \Sigma}) =  \norm{f}_{\bs  \Sigma}^2$, in practical settings the true $f$ is unknown. 
Instead, we have access to a finite dataset $\{(\bs x^i, f^i)\}_{i=1}^M$. We therefore optimize an empirical surrogate objective $g_M({\bs  \Sigma})$, such as the empirical RKHS norm of the minimum-norm interpolator, 
\begin{equation}\label{eq:g_M}
g_M({\bs  \Sigma}) \coloneqq \bs f_M^\top \bs K_{{\bs  \Sigma},M}^{-1}\bs f_M, \quad  \text{ where }  \bs K_{{\bs  \Sigma},M} = \big(\kappa_{{\bs  \Sigma}}(\bs x^i,\bs x^j)\big)_{i,j=1}^M.
\end{equation}
Assuming the samples $\bs x^i$ are distinct (which holds almost surely for a continuous measure $\mu$), the kernel matrix is strictly positive definite and invertible, meaning $g_M$ is well-defined.
A fundamental question in kernel learning is whether the data-driven parameters 
\begin{equation}\label{eq:Sigma_hat}
\hat{{\bs  \Sigma}} \coloneqq \argmin_{{\bs  \Sigma} \in (0,1)^d} g_M({\bs  \Sigma})
\end{equation}
reliably approximate the theoretical optimum ${\bs  \Sigma}^*$, and whether the resulting approximation error remains small. The strong convexity of $g({\bs  \Sigma})$ proven in Theorem~\ref{thm:convexity} provides the theoretical guarantee for this parameter stability. 
The following theorem shows uniform convergence of $g_M$ to the true function $g$.

\begin{theorem}\label{thm:unif_conv}
Let $\{(\bs x^i,f^i)\}_{i\in\N}$ be a sequence where each $\bs x^i\in\D^d$ is sampled i.i.d. from a distribution $\mu$, whose support is $\D^d$. Here, $f^i=f(\bs x^i)$. Let $g_M({\bs  \Sigma})$ be defined as in~\eqref{eq:g_M}. Define $\bs f^M$ to be the vector of length $M$ containing the first $M$ values $f^i=f(\bs x^i)$. The matrix $\bs K_{{\bs  \Sigma},M}$ is the kernel matrix from the first $M$ points $\bs x^i$ for the kernel $\kappa_{\bs  \Sigma}$. Let $g$ be as defined in Theorem~\ref{thm:convexity} as $g({\bs  \Sigma})=\|f\|_{\bs  \Sigma}^2$. Then $g_M$ converges uniformly to $g$ on the domain $[\eta,1-\eta]^d$ for any $\eta\in(0,1)$. 
\end{theorem}
\begin{proof} The structure of $g_M$ and $g$ implies they are continuous.
    For any ${\bs  \Sigma}$, the evaluation of $g_M({\bs  \Sigma})$ is equal to 
    \[
    \min_{\bs \alpha\in\R^M}\bs\alpha^\top \bs K_{{\bs  \Sigma},M}\bs\alpha\quad \text{subject to}\ \bs K_{{\bs  \Sigma},M}\bs\alpha=\bs f^M.
    \]
    This is the smallest norm possible for a function in $\mathcal{H}_{\bs  \Sigma}$ that interpolates $\{(\bs x^i,f^i)\}_{i=1}^M$. In other words, it is equal to 
    \[
    \min_{\hat f\in\mc H_{\bs  \Sigma}}\|\hat f\|_{\bs  \Sigma}^2\quad \text{subject to}\ \hat f(\bs x^i)=f^i,\ i=1,\dots,M.
    \]
    Since $f$ also interpolates the data, $\|\hat f_{{\bs  \Sigma},M}\|_{\bs  \Sigma}\leq \|f\|_{\bs  \Sigma}$, where $\hat f_{{\bs  \Sigma},M}$ is the norm-minimizing interpolating function. Then, $g_M({\bs  \Sigma})\leq g({\bs  \Sigma})$ for all ${\bs  \Sigma}.$ In fact, the subset of interpolating functions in $\mathcal{H}_{\bs  \Sigma}$ shrinks as we increase the number of points to interpolate. Then, $g_M({\bs  \Sigma})\leq g_{M+1}({\bs  \Sigma})$ and the sequence $\{g_M({\bs  \Sigma})\}_{M\in\N}$ is monotone increasing. 
    
    We show that $g_M({\bs  \Sigma})\to g({\bs  \Sigma})$ pointwise. 
    Each pointwise sequence $g_M({\bs  \Sigma})$ is monotone increasing and bounded above by $g({\bs  \Sigma})$. Then, $g_M({\bs  \Sigma})\to g_{\bs  \Sigma}^*$, with $g_{\bs  \Sigma}^*\leq g({\bs  \Sigma})$. 
    Since $\{\hat f_{{\bs  \Sigma},M}\}_{M\in\N}$ is uniformly bounded in $\mathcal{H}_{\bs  \Sigma}$, there is a  subsequence $\{\hat f_{{\bs  \Sigma},M_i}\}_{i\in\N}$ that weakly converges to some $f^*_{\bs  \Sigma}$ in $\mathcal{H}_{\bs  \Sigma}$. As $\mathcal{H}_{\bs  \Sigma}$ is an RKHS,  $\langle h,\kappa_{\bs  \Sigma}(\cdot,\bs x)\rangle_{\mathcal{H}_{\bs  \Sigma}}=h(\bs x)$ for all $h\in\mathcal{H}_{\bs  \Sigma}$ and $\bs x\in\D^d$. Weak convergence implies pointwise convergence of the subsequence, 
    $\lim_{i \to \infty} \langle \hat{f}_{{\bs \Sigma},M_i}, \kappa_{\bs \Sigma}(\cdot, \bs x)\rangle = \langle f^*_{\bs \Sigma}, \kappa_{\bs \Sigma}(\cdot, \bs x)\rangle$.
    This implies \[\|f^*_{\bs  \Sigma}\|_{\bs  \Sigma}^2 \leq \liminf_{i\to\infty}\|\hat f_{{\bs  \Sigma},M_i}\|_{\bs  \Sigma}^2\leq g_{\bs  \Sigma}^*.\]
 Suppose $\{\bs x^i\}_{i=1}^\infty$ is dense in $\D^d,$ which is true almost surely. Because $f^*_{\bs  \Sigma}(\bs x^i) = f^i$ for all $i$, $f^*_{\bs  \Sigma}$ and $f$ are equal to each other on a dense subset of $\D^d$. 
Since the kernel $\kappa_{\bs \Sigma}$ is continuous, all functions in the RKHS $\mathcal{H}_{\bs \Sigma}$ are continuous. Continuity of both $f$ and $f^*_{\bs \Sigma}$ implies $f^*_{\bs  \Sigma}=f$. So, $g({\bs  \Sigma})\leq g_{\bs  \Sigma}^*$, meaning $g({\bs  \Sigma})= g_{\bs  \Sigma}^*$.

    Because all $g_M$ and $g$ are continuous and the sequence is monotone increasing and converges pointwise, Dini's theorem~\cite{Rudin} yields uniform convergence $g_M\to g$ on any compact subset of $(0,1)^d$.
\end{proof}
The previous theorem shows that the empirical surrogate $g_M$ uniformly approximates the true objective $g$. A direct consequence is that the error between the TSK factors from the empirical minimizing problem and the true minimizer of the objective $g$ is bounded, as the following corollary shows.
\begin{corollary}
   By Theorem~\ref{thm:unif_conv}, for every $M\in\N$ there exists $\epsilon_M > 0$ such that
    \begin{equation*}
        \sup_{\bs \Sigma\in [\eta,1-\eta]^d} \vert g(\bs \Sigma) - g_M(\bs \Sigma) \vert \leq \epsilon_M,
    \end{equation*}
   where $\eta\in(0,1)$. Suppose $g$ is $m$-strongly convex, as given by Theorem~\ref{thm:convexity}. 
    Then, the error between the TSK factors $\hat{\bs \Sigma}_M$ derived by the empirical minimization~\eqref{eq:Sigma_hat} and the true minimizer ${\bs \Sigma}^*$ of the function $g$ is bounded by
    \[
    \norm{\hat{{\bs \Sigma}}_M - {\bs \Sigma}^*} \leq 2\sqrt{\frac{\epsilon_M}{m}}.
    \]
\end{corollary}
\begin{proof}
Using Theorem~\ref{thm:convexity} and the definition of strong convexity evaluated at the minimum ${\bs \Sigma}^*$ (see \cite{Ne04}), there exists a constant $m > 0$ such that for any ${\bs \Sigma} \in [\eta,1-\eta]^d$,
\[
\frac m2 \norm{{\bs \Sigma} - {\bs \Sigma}^*}^2 \leq g({\bs \Sigma}) - g({\bs \Sigma}^*).
\]
Evaluating this at the empirical minimizer $\hat{{\bs \Sigma}}_M$ and utilizing the uniform bound yields
\begin{align*}
\frac{m}{2} \Vert\hat{{\bs \Sigma}}_M - {\bs \Sigma}^*\Vert^2 &\leq g(\hat{{\bs \Sigma}}_M) - g({\bs \Sigma}^*) 
= g(\hat{{\bs \Sigma}}_M) - g_M(\hat{{\bs \Sigma}}_M) + g_M(\hat{{\bs \Sigma}}_M) - g({\bs \Sigma}^*) \\
&\leq g(\hat{{\bs \Sigma}}_M) - g_M(\hat{{\bs \Sigma}}_M) + g_M({\bs \Sigma}^*) - g({\bs \Sigma}^*) 
\leq \vert g(\hat{{\bs \Sigma}}_M) - g_M(\hat{{\bs \Sigma}}_M) \vert + \vert g_M({\bs \Sigma}^*) - g({\bs \Sigma}^*) \vert \\
&\leq 2 \sup_{{\bs \Sigma} \in [\eta,1-\eta]^d} \vert g_M({\bs \Sigma}) - g({\bs \Sigma})\vert \leq 2\epsilon_M.
\end{align*}
Here, the first inequality relies on the definition of the empirical minimizer, which guarantees $g_M(\hat{{\bs \Sigma}}) \leq g_M({\bs \Sigma}^*)$. Rearranging this inequality directly provides a bound on the parameter estimation error,
\[
\norm{\hat{{\bs \Sigma}}_M - {\bs \Sigma}^*} \leq 2\sqrt{\frac{\epsilon_M}{m}}.
\]
This finishes the proof.
\end{proof}
The previous corollary demonstrates the critical role of strong convexity in ensuring parameter identifiability and stability. Because the parameters $\hat{{\bs \Sigma}}$ converge to ${\bs \Sigma}^*$ at a rate governed by the statistical estimation error $\epsilon_M$, and because the kernel mapping ${\bs \Sigma} \mapsto \kappa_{\bs \Sigma}$ is smooth, the learned kernel $\kappa_{\hat{{\bs \Sigma}}}$ will uniformly converge to the optimal kernel $\kappa_{{\bs \Sigma}^*}$. Consequently, the approximation error of the model trained with the learned parameters $\hat{{\bs \Sigma}}$ will only deviate from the optimal approximation error by a term proportional to $\mathcal{O}(\sqrt{\epsilon_M})$. This confirms that optimizing the TSK factors $\bs \Sigma$ from finite data is mathematically well-posed and preserves the optimal generalization properties of the minimum-norm interpolator.

At present, we cannot guarantee convexity of the finite-data objective, unlike the exact version~\eqref{eq:tsk}. We speculate that the finite-data objective may be convex with high probability under random data sampling. Empirical evidence supports this. In numerical experiments, we have computed the analytical Hessian of $g_M$ at randomly sampled values of $\bs\Sigma$ throughout $(0,1)^d$ and consistently observe it to be positive definite.\\

\textbf{Connection to automatic relevance determination}\\
The finite-data objective~\eqref{eq:g_M} is related to kernel hyperparameter learning by maximizing Gaussian process marginal likelihood. Automatic relevance determination (ARD) adapts a kernel through input-specific length scales~\cite[Section 5.1]{williams2006gaussian}. For a Gaussian kernel, ARD takes the form
\begin{equation*}
    \kappa_{\bs\ell}(\bs x,\bs y)
    =
    \e^{
        -\frac{1}{2}\sum_{k=1}^d
        (x_k-y_k)^2/\ell_k^2
    },
    \qquad \ell_k>0.
\end{equation*}
A larger value of $\ell_k$ makes the kernel less sensitive to variation in the $k$th input. Assuming the data $\bs f$ are noise-free, maximizing the Gaussian process marginal likelihood is equivalent to
\begin{equation}\label{eq:ard_objective}
    \bs\ell^*
    =
    \argmin_{\bs\ell>0}
        \bs f^\top \bs K_{\bs\ell}^{-1}\bs f
        + \log(\det \bs K_{\bs\ell}).
\end{equation}
The first term in~\eqref{eq:ard_objective} is the squared RKHS norm of the minimum-norm interpolant and is therefore analogous to the finite-data TSK objective
$g_M(\bs\Sigma)=\bs f^\top \bs K_{\bs\Sigma}^{-1}\bs f$.
The two approaches nevertheless adapt the approximation space in different ways. ARD rescales the inputs of the kernel and includes the log-determinant term arising from marginal likelihood. On the other hand, TSK changes the relative weights of the ANOVA component spaces and selects these weights solely through RKHS norm minimization. 

\subsection{The Algorithm}
Our proposed learning algorithm consists of two steps: 1) TSK factor optimization; 2) kernel ridge regression. TSK factor optimization consists of finding a minimizer of the finite-data objective~\eqref{eq:g_M}. The gradient and Hessian of the objective have explicit expressions. We opt for quasi-Newton methods to save on cost. Particularly, we use the L-BFGS algorithm~\cite{lbfgs,lbfgs2016} to minimize the finite-data objective. Our method is summarized in~\cref{alg:tsk}.
 \begin{algorithm}[htb]
    \caption{Total sensitivity kernel learning}
    \begin{flushleft}
   \textbf{Inputs}: Data $\{(\x^i,f^i)\}_{i=1}^M$ with data vector $\bs f^M = (f^i)_{i=1}^M$, characteristic functions $\phi_k$, regularization parameter $\lambda$, initialization $\bs\Sigma^{(0)}$\\
\textbf{Outputs}: TSK approximation $f^*$, TSK factors $\bs\Sigma^*$
\end{flushleft}
    \begin{algorithmic}[1]
    \State Set initial TSK factors to ${\bs  \Sigma}^{(0)}$ 
    \State Precompute, for arbitrary TSK factors $\bs \Sigma$, the kernel matrix by
    \[(\bs K_{\bs \Sigma,M})_{i,j} = \prod_{k=1}^d\big(1-\Sigma_k + \Sigma_k\phi_k(x_k^i - x_k^j)\big)\]
    \State Compute minimizer $\bs\Sigma^* = \argmin_{\bs\Sigma} g_M({\bs  \Sigma}) = (\bs f^M)^\top \bs K_{{\bs  \Sigma},M}^{-1}\bs f^M$
    \State Solve $f^* = \argmin_{\hat f \in {\mc H}_{\bs \Sigma}} \frac{1}{M} \sum_{i=1}^M \big(\hat f(\bs x^i) - f^i\big)^2 + \lambda \norm{\hat f}_{\bs \Sigma}^2$
    \end{algorithmic}\label{alg:tsk}
\end{algorithm}
We must address some numerical concerns that can cause the algorithm to fail.

\subsubsection*{Reparameterization}
The optimization takes place over TSK factors $\bs\Sigma\in(0,1)^d$. To ensure that the iterates remain within this admissible region, we use a sigmoid reparameterization,
\begin{equation*}
    \Sigma_k = s(z_k) = \frac{1}{1+\e^{-z_k}}, \qquad z_k\in\mathbb{R}.
\end{equation*}
The optimization is then performed over the unconstrained variables $\bs z\in\mathbb{R}^d$. This prevents optimization iterates from leaving the region where the TSK factors are defined. Centering and scaling the input data also improves the numerical behavior of the kernel matrix. For kernels like the Gaussian kernel, differences in the magnitudes of the inputs can cause excessively small or disparate kernel values. Centering and scaling ensure the data are not spaced too far apart.

\subsubsection*{Initialization matters}
Another issue is that optimization can stall when the input dimension is high.
This results from how the input dimension impacts the structure of the kernel matrix in~\eqref{eq:g_M}.
Recall the kernel matrix has entries
\begin{equation*}
    (\bs K_{\bs\Sigma,M})_{i,j}
    =
    \prod_{k=1}^d
    \big(1-\Sigma_k+\Sigma_k\phi_k(x_k^i-x_k^j)\big).
\end{equation*}
Each entry is a product of convex combinations. Since each $\phi_k$ is a characteristic function, the diagonal entries are always equal to $1$. The off-diagonal entries are a different story, as they correspond to pairs of distinct data points. Consider the standard kernel, when all TSK factors equal $1$, whose entries are just products of characteristic functions with $|\phi_k|\leq 1$. For some characteristic functions commonly used in kernel learning, like the Gaussian characteristic function, the value decays to $0$ as the distance between the two data points increases. Off-diagonal entries for the standard kernel are then products of $d$ values with magnitude less than $1$. Larger input dimensions $d$ accelerate these products towards $0$. The result is a kernel matrix that, numerically speaking, approaches the identity matrix. This phenomenon can occur for TSK factor values near $1$, too, such as $0.9$ or $0.95$.

The same phenomenon affects the gradient of the finite-data objective. The gradient is given by entries
\begin{equation*}
    \frac{\partial g_M}{\partial \Sigma_\ell}
    =
    -\bs\alpha^\top
    \frac{\partial \bs K_{\bs\Sigma,M}}{\partial\Sigma_\ell}
    \bs\alpha,\quad 
    \frac{\partial(\bs K_{\bs\Sigma,M})_{i,j}}
    {\partial\Sigma_\ell}
    =
    \big(\phi_\ell(x_\ell^i-x_\ell^j)-1\big)
    \prod_{k\neq\ell}
    \big(1-\Sigma_k+\Sigma_k\phi_k(x_k^i-x_k^j)\big).
\end{equation*}
The diagonal entries of these derivative matrices are $0$. When the TSK factors are near $1$, the off-diagonal entries can also become numerically close to $0$ in high dimensions because they contain the same type of product as the kernel matrix. The gradient of $g_M$ can therefore become numerically small even when the TSK factors are far from the minimizer. The effect is that, numerically speaking, the function $g_M(\bs\Sigma)$ appears nearly constant in the region of near-$1$ TSK factors, and optimization stalls.

One must be careful about where to initialize the optimizer. We initialize the TSK factors away from~$1$. Unless otherwise stated, we take
\begin{equation*}
    \Sigma_{k}^{(0)}=0.2,\qquad k=1,\ldots,d.
\end{equation*}
This appears to mitigate the vanishing of the off-diagonal entries of the kernel matrix and its derivatives. The corresponding initialization for the unconstrained optimization variables is
\begin{equation*}
    z_{k}^{(0)}=\log\Big(\frac{\Sigma_{k}^{(0)}}{1-\Sigma_{k}^{(0)}}\Big).
\end{equation*}

\section{Numerical experiments}\label{sec:numerics}
The selected numerical experiments display how TSKs exploit multivariable structure. 
 To optimize TSK factors, we initialize each one at a value of $\Sigma_{k}^{(0)}=0.2$. Optimization uses the quasi-Newton L-BFGS method. 
The exact gradient is supplied. This methodology holds for all experiments. All validation sets use $10^4$ points and the same distribution as the training data.

We compare our method TSK with the standard tensor product kernel 
\begin{equation}
\kappa_{\bs 1}(\bs x,\bs y)=\prod_{k=1}^d\phi_k(x_k-y_k),
\end{equation}
and the ANOVA kernel~\eqref{equ:anova_kernel} with the same one-dimensional kernels as for TSK. Additionally, for the Gaussian kernels we also compare with ARD, described in~\eqref{eq:ard_objective}.

\subsection{Case study: Sobol' g-function}
The Sobol' g-function, from~\cite{Saltelli10}, appears frequently as a sensitivity analysis benchmark. It typically takes 8 inputs
\begin{equation}\label{equ:gfcn}
f(\bs{x})=\prod_{k=1}^8 \frac{|4x_k-2|+a_k}{1+a_k}, 
\end{equation}
where $\bs x\sim\mathcal{U}([0,1]^8).$ The coefficient $a_k$ determines the importance of the corresponding input. A greater value of $a_k$ decreases the importance of the input $x_k$: $x_k$ is very important when $a_k\approx 0$, is of lesser importance when $a_k\approx 9$, and makes almost no impact when $a_k\approx 99$~\cite{Marrel08}. This control over the input structure makes the g-function an ideal case study. We choose three sets of coefficients $\bs a_{\text{high}}$, $\bs a_{\text{med}}$ , and $\bs a_{\text{low}}$ to create high, medium, and low complexity input interaction structures, 
\begin{align*}
\bs a_{\text{high}} = &\left[\begin{array}{cccccccc}
0 & 0 & 1 & 1 & 2 & 2 & 3 & 3
\end{array}\right], \quad &\text{type A}, \\
\bs a_{\text{med}} =& \left[\begin{array}{cccccccc}
0.5 & 0.5 & 7& 7 & 7 & 7 & 7 & 7
\end{array}\right], \quad &\text{type B},\\
\bs a_{\text{low}} =& \left[\begin{array}{cccccccc}
1 & 2 & 5 & 10 & 20 & 50 & 100 & 500
\end{array}\right], \quad &\text{type C}.
\end{align*}
The case study demonstrates how TSKs identify and take advantage of structurally low complexity functions. A function is easy or difficult to approximate depending on its ANOVA decomposition. This is formalized with notions of type A, B, and C functions in~\cite{kucherenko2011identification}. A function is complicated (or type A) and challenging to approximate, when no inputs are unimportant and all interactions, low-order to high-order, are important. A function of medium complexity (or type B) has no non-influential inputs, but is dominated by main effects and lower-order interaction terms. The low complexity (or type C) case describes functions where some inputs are unimportant and only lower-order interactions and main effects have influence. Each set of coefficients $\bs a_{\text{high}}$, $\bs a_{\text{med}}$, and $\bs a_{\text{low}}$ falls into either type A, B, or C. \par
We run Algorithm~\ref{alg:tsk} with $M=10^3$ samples to create kernel approximations of each case of the g-function~\eqref{equ:gfcn}. Because of the nonsmoothness of the g-function, the kernels use the characteristic function $\phi(t) = \e^{-|t|}$, which corresponds to the Cauchy distribution. 
 \begin{figure}[h!]
      \centering
      \includegraphics[width=.32\linewidth]{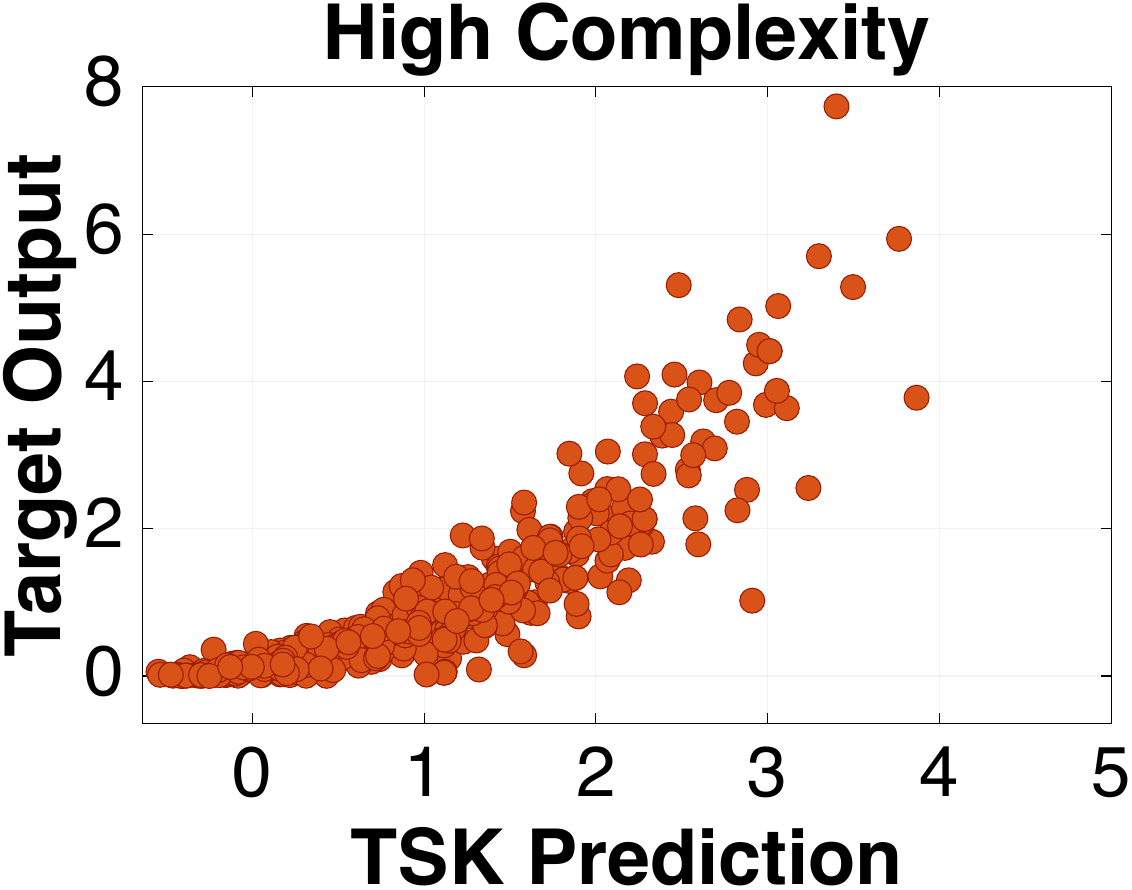}
      \includegraphics[width=.32\linewidth]{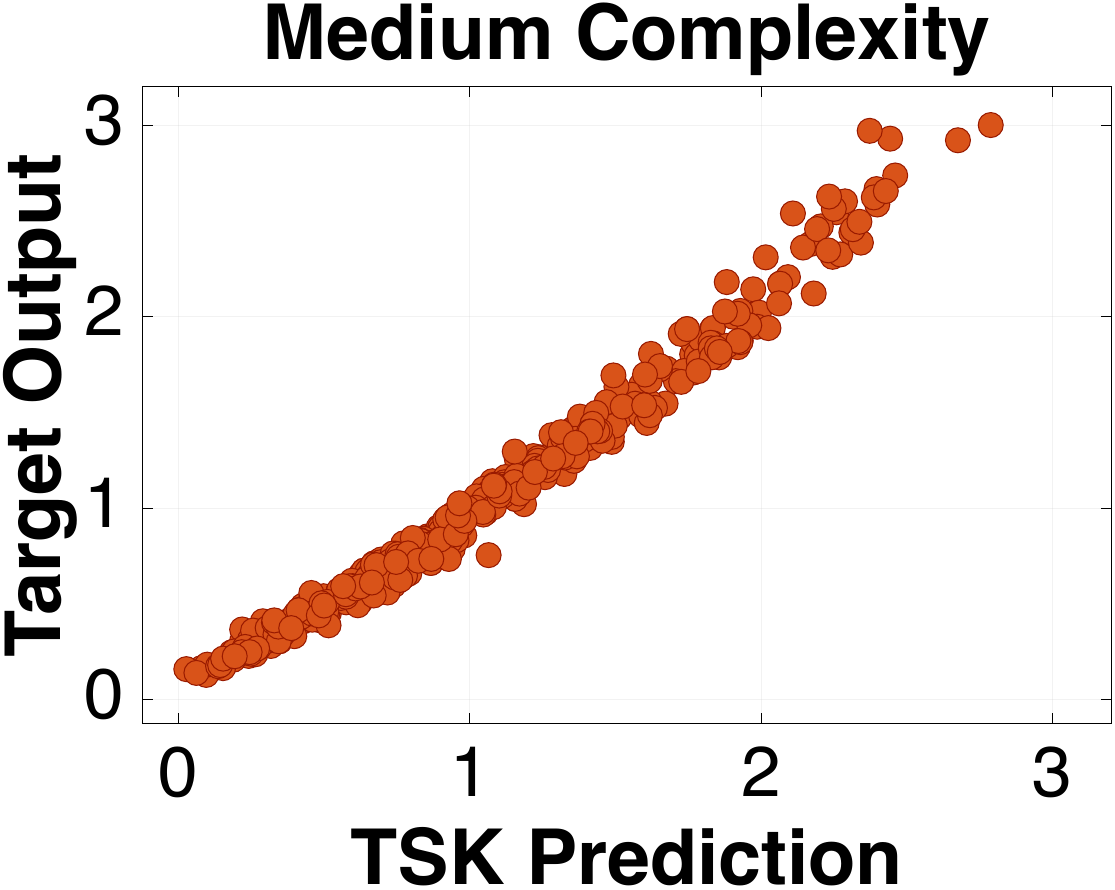}
          \includegraphics[width=.32\linewidth]{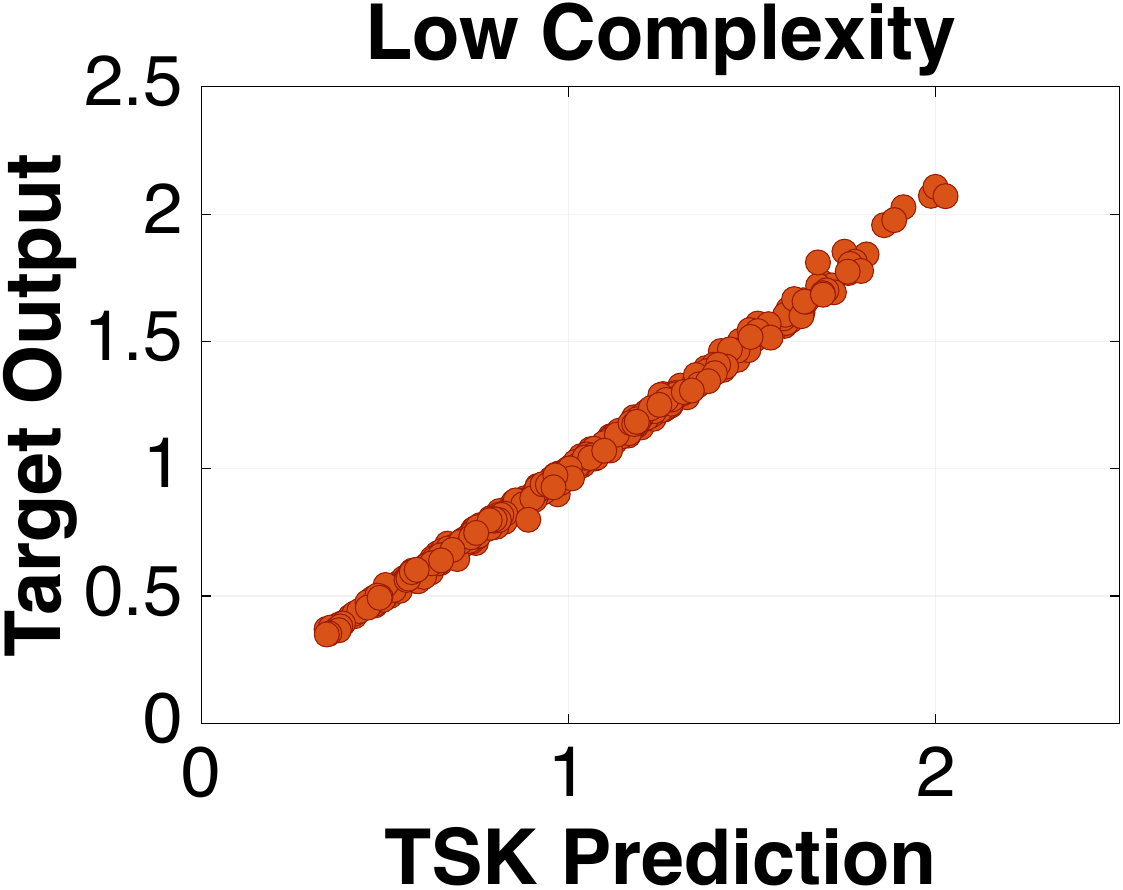}
                \includegraphics[width=.32\linewidth]{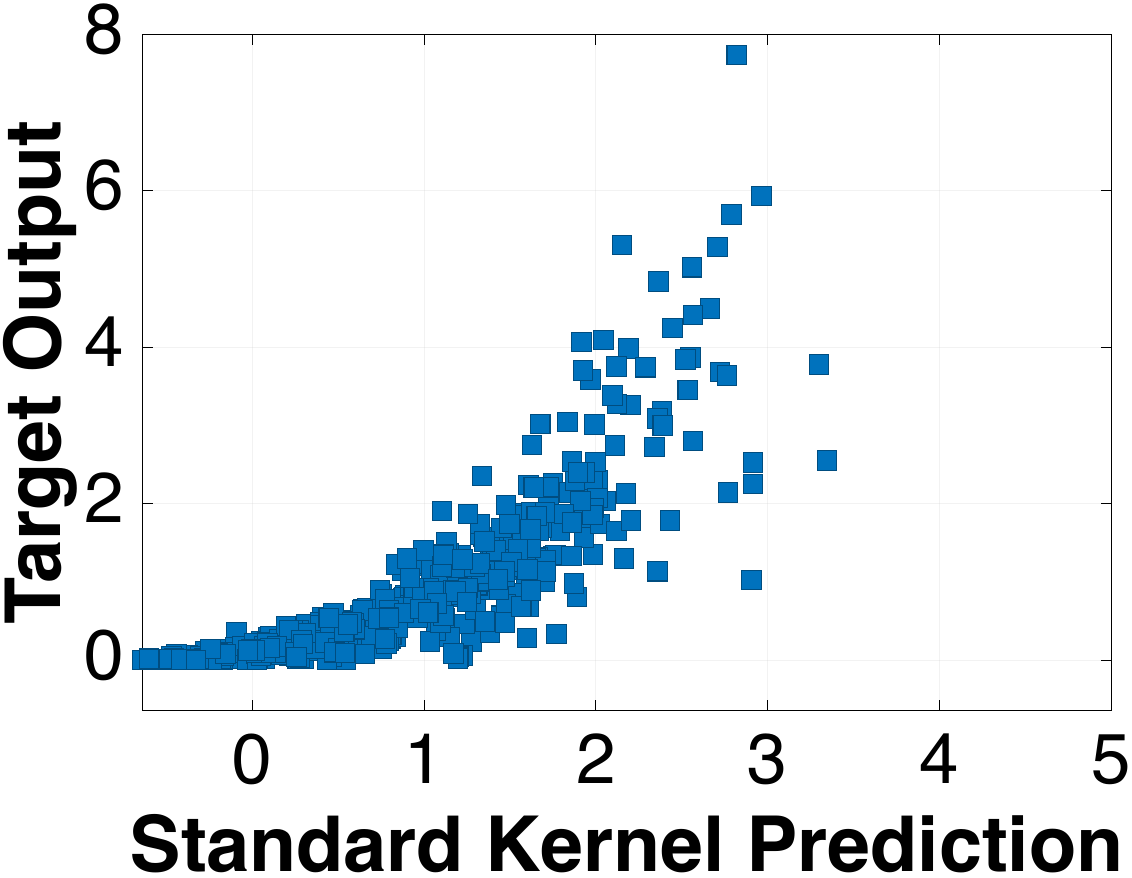}
      \includegraphics[width=.32\linewidth]{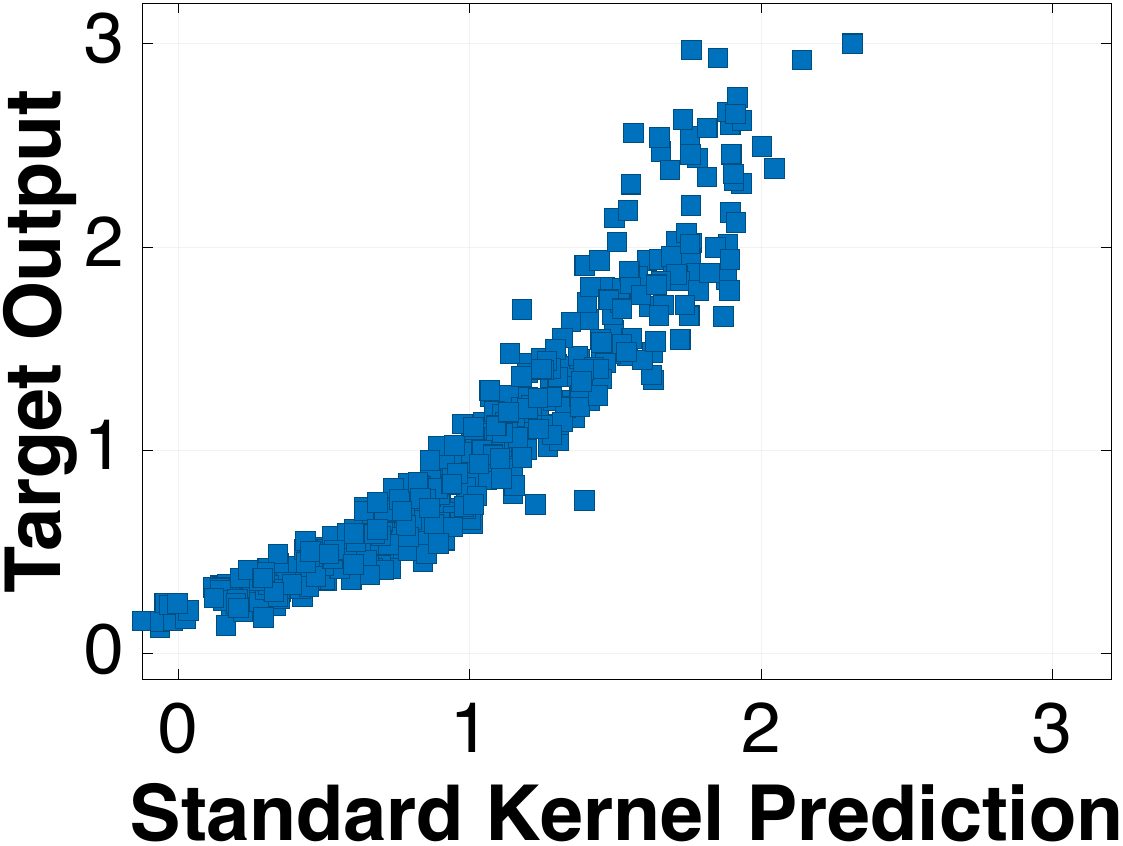}
          \includegraphics[width=.32\linewidth]{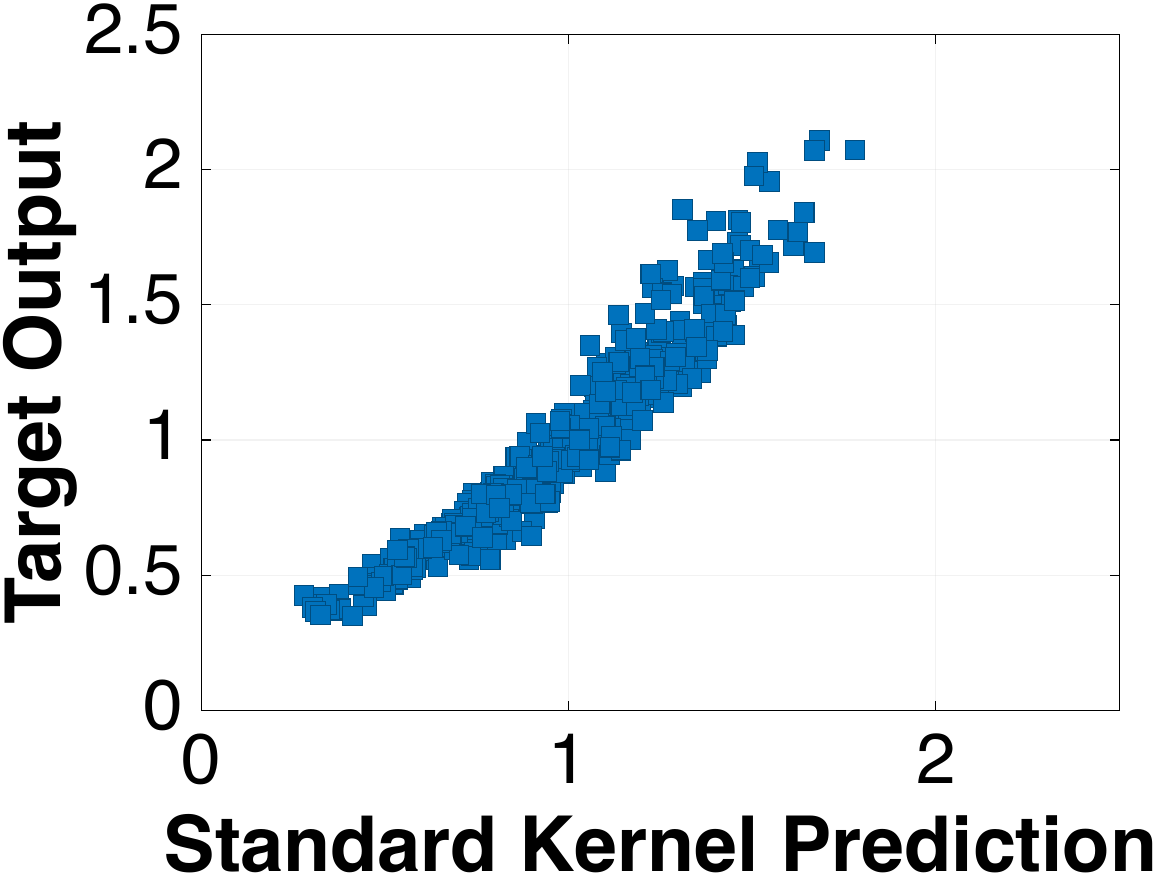}
                          \includegraphics[width=.32\linewidth]{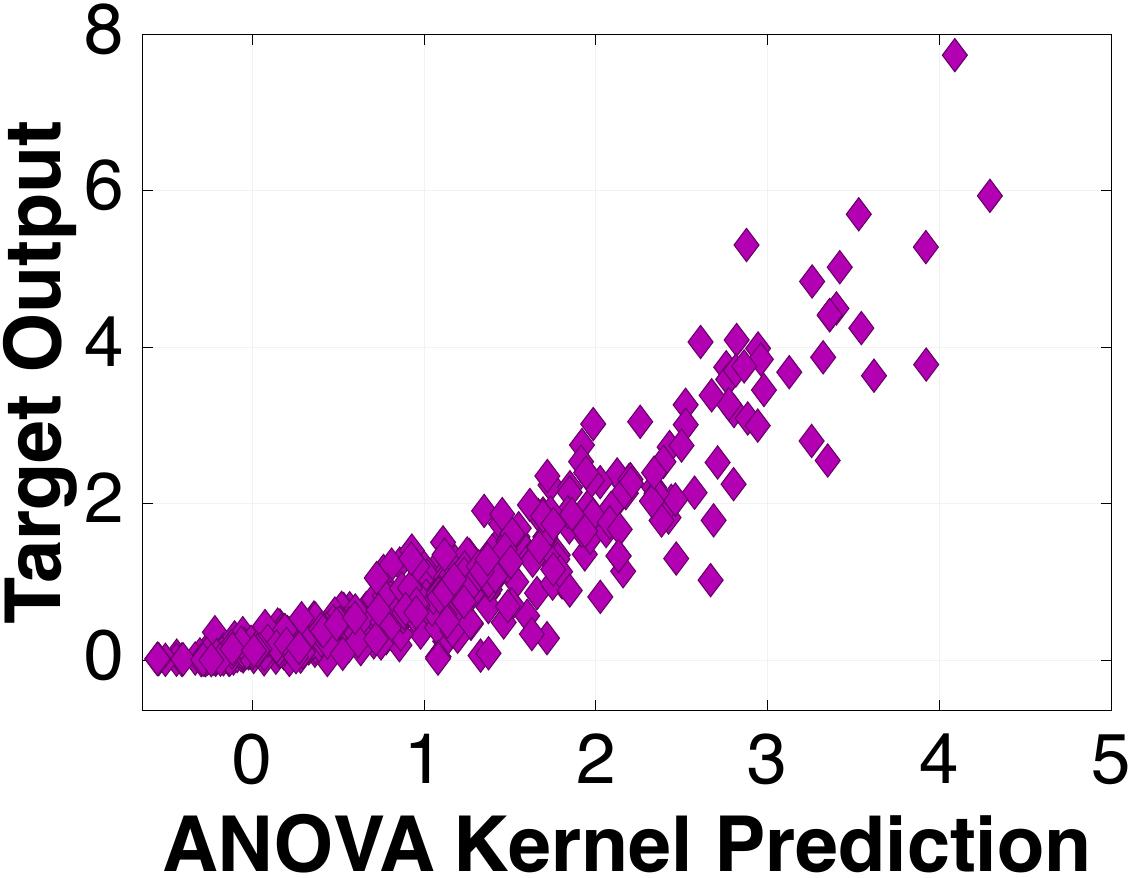}
      \includegraphics[width=.32\linewidth]{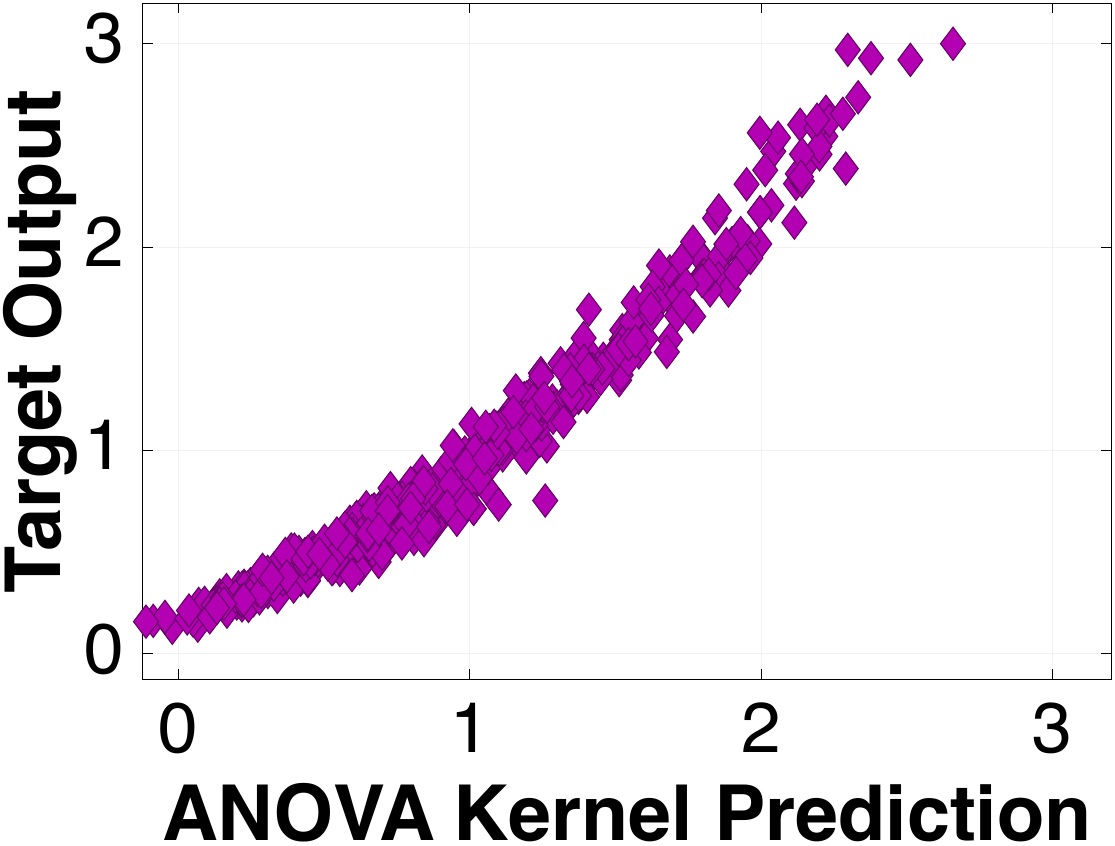}
 \includegraphics[width=.32\linewidth]{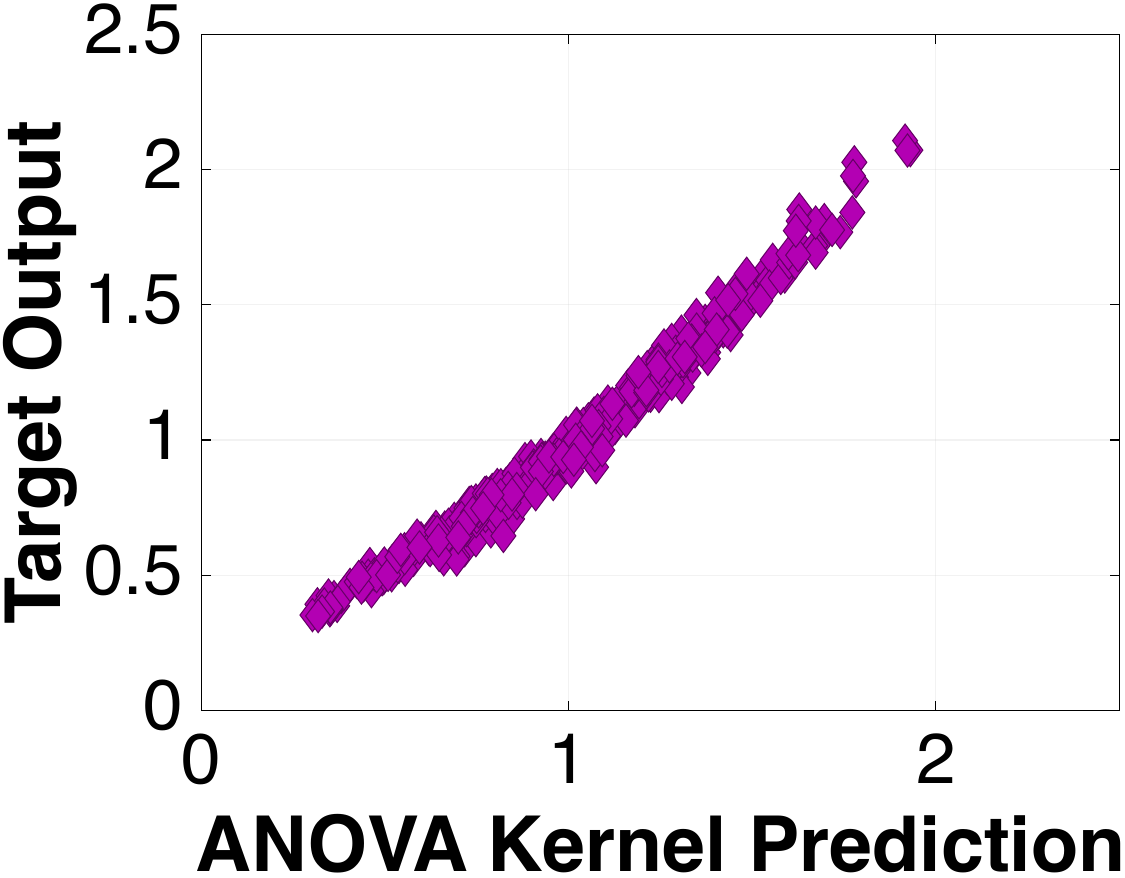}
            \includegraphics[width=.32\linewidth]{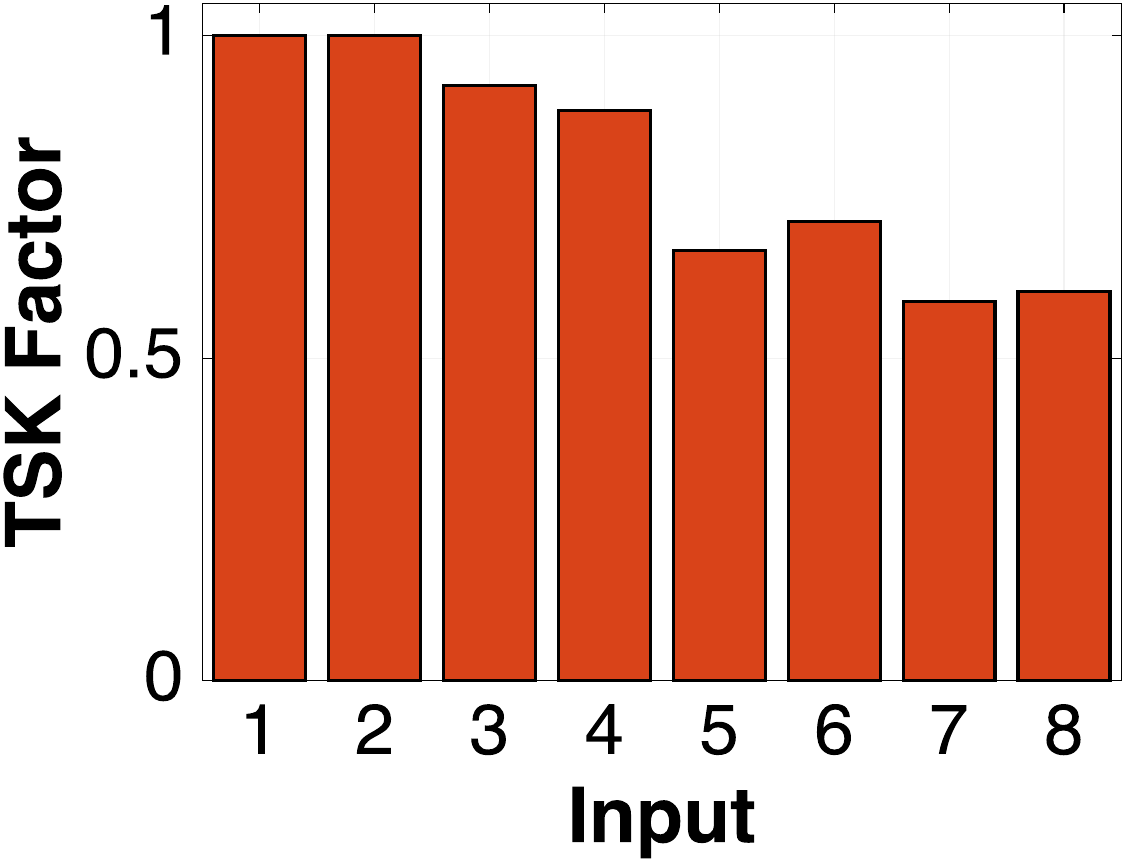}
      \includegraphics[width=.32\linewidth]{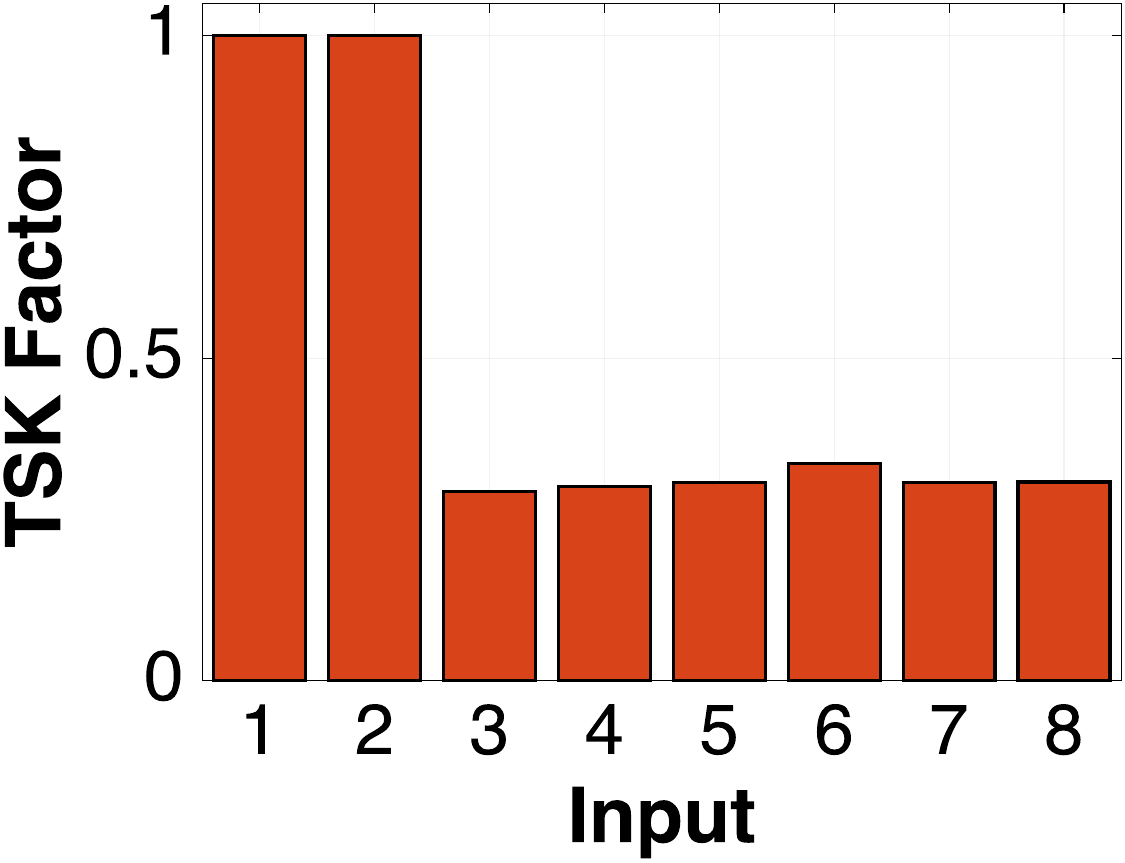}
        \includegraphics[width=.32\linewidth]{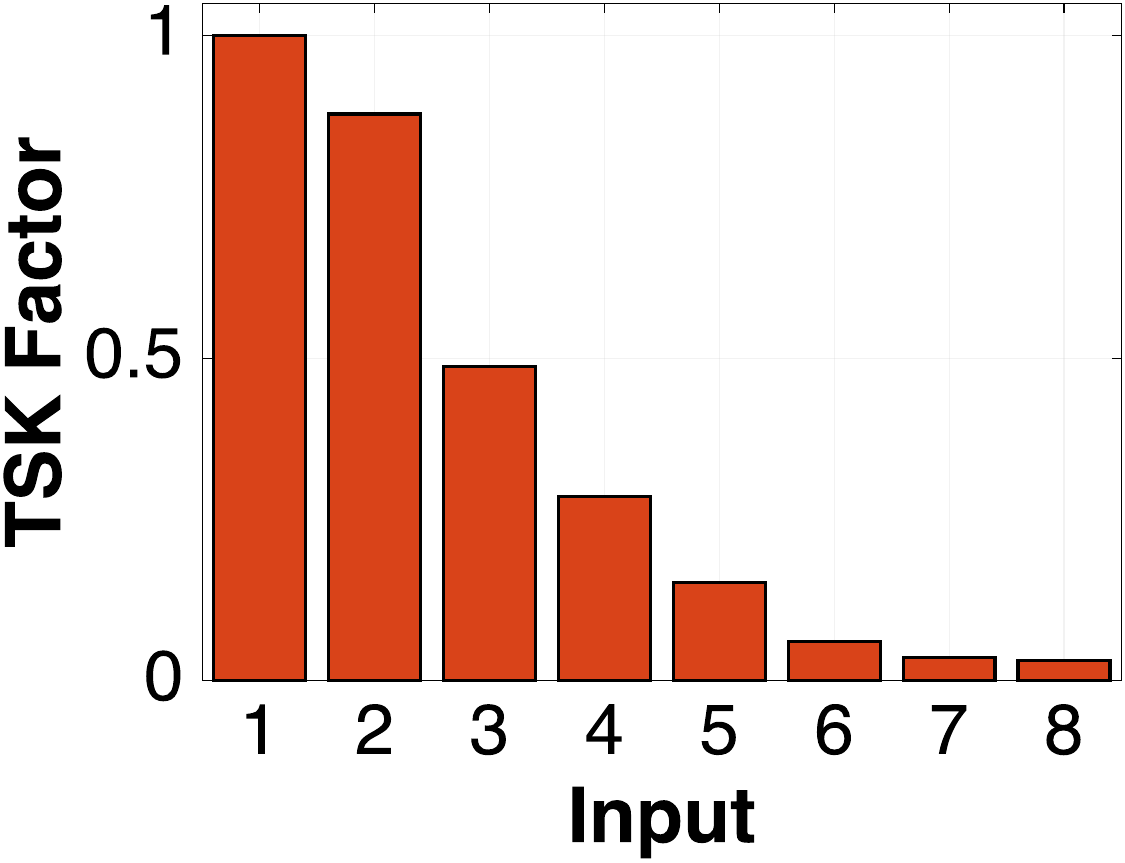}  
      \caption{Results for approximation of the g-function~\eqref{equ:gfcn} with TSK (top row), standard product kernels (second row), and unweighted ANOVA kernels (third row). The left, middle, and right columns correspond to high, medium, and low complexity structured functions, respectively. We show correlations between the kernel predictions and true outputs for $500$ of the $10^4$ validation points. Strong correlation, and more accurate prediction, occurs when the scatter plot is closer to a line with slope one. The bottom row  presents the TSK factors taken after optimization.}
      \label{fig:gfunc}
  \end{figure}
The results for all cases appear in~\cref{fig:gfunc}. Scatter plots show the true output values versus what the kernel predicts for each point in the validation set. We present the optimized TSK factors in bar graphs.\par
  
Root mean squared error (RMSE) and root relative squared error (RRSE) are evaluated on a validation set of $10^4$ points sampled via LHS. The errors are given in Table~\ref{tab:gfunc_err}. 
\begin{table}[h!]
    \centering
    \begin{tabular}{|cc|c|c|c|} 
    \hline
   Complexity & Error Type & TSK & Standard Ker. & ANOVA Ker.\\ \hline
   High & RMSE & 0.565 & 0.671 & $\mathbf{0.507}$ \\
     & RRSE & $0.367$ & 0.437 & $\mathbf{0.330 } $
      \\ \hline
     Medium & RMSE & $\mathbf{9.570\times 10^{-2}}$ & $0.233$ & $0.139$  \\
     & RRSE & $\mathbf{8.166\times 10^{-2}}$ & $0.199$ & $0.119$ \\ \hline
     Low & RMSE & $\mathbf{2.571\times 10^{-2}}$ & $0.128$ & $6.008\times 10^{-2}$ \\
     & RRSE& $\mathbf{2.412\times 10^{-2}}$ & 0.120 & $5.637\times 10^{-2}$  \\ \hline
    \end{tabular}
    \label{tab:gfunc_err}
    \caption{Root mean squared error (RMSE) and root relative squared error (RRSE) for kernel machines of the TSK, standard kernel, and ANOVA kernel on a validation set. Validation sets consist of $10^4$ data points for the high, medium, and low complexity g-function.}
\end{table}

\raggedbottom
In each complexity case, TSK shows better performance than the standard kernel and performs roughly similar to or better than the ANOVA kernel. Decreasing complexity leads to better performance regardless of the kernel used. The least complex function is easier to approximate. However, TSK appears better able to take advantage of decreasing complexity. In the high complexity case, the performance gap is smaller. We see the optimal TSK factors in~\cref{fig:gfunc} (left column) identify all inputs as important, as all factors take values tending towards 1 rather than 0. The optimized kernel is not so different from the standard kernel. Visually, in~\cref{fig:gfunc} (left column), the TSK predictions do not correlate much better with the true values than those of the standard kernel. The situation changes with the medium complexity case where the TSK performs twice as well as the standard kernel. There is a starker difference in correlation of the predictions in~\cref{fig:gfunc} (middle column). There is also more disparity in the TSK factors, which aligns with the choice of $\bs a_{\text{med}}$. For the low complexity case, the TSK error is less than one fourth of the standard-kernel error and half of the ANOVA-kernel error. In~\cref{fig:gfunc} (right column), the TSK predictions correlate almost perfectly. 

The experiment highlights how the optimized kernel performs better as the standard kernel across all cases. Key is how the performance of the optimized kernel $\kappa_{\bs \Sigma^*}$ improves over the standard kernel and ANOVA kernel when the function's interaction structure becomes simpler. This demonstrates how optimizing TSK factors allows the method to identify and take particular advantage of function structure. The ANOVA kernel performs better than the standard kernel by its greater expression of structure. Unlike TSK, the lack of weighting in the ANOVA kernel limits performance.

\subsection{Higher-dimensional examples}
The following experiments investigate how beneficial it is to adapt the kernel to the multivariable structure as the input dimension increases and the number of available function evaluations remains limited. The kernels are Gaussian, using the characteristic function $\phi(t) = \e^{-t^2/2}$, which corresponds to the standard normal distribution.
High dimensions present a particular difficulty for the standard product kernel. As discussed in~\cref{sec:alg}, high input dimension results in a kernel matrix $\kappa_{\bs 1}$ near the identity matrix. When evaluating on the validation set, we construct a prediction kernel matrix built from both the training and validation data. Since the training and validation sets almost surely do not share any data, the prediction kernel matrix can be nearly a zero matrix. The ANOVA and TSK kernels alter this product structure by including lower-dimensional component kernels. Use of Gaussian kernels enables comparison to automatic relevance determination (ARD), which can mitigate the effect by adapting coordinate-wise length scales. The experiments below examine these different mechanisms on functions possessing varying forms of multivariable structure.
 We study four examples, two of which have smooth analytic expressions.
\subsubsection*{100D function}
The 100D function is a benchmark studied in~\cite{Luthen21}. The model's expression is
\begin{equation}\label{equ:100d}
    f(\bs x) = 3 +\frac{1}{100}\sum_{k=1}^{100} k\big(x_k^3-5x_k+\frac{1}{3}\log(x_k^2+x_k^5)\big) + x_1x_2^2 + x_2x_4 -x_3x_5+x_{51} +x_{50}x_{54}^2,
\end{equation}
where $x_{20}\sim\mathcal{U}([1,3])$ and $x_k\sim\mathcal{U}([1,2])$ for $k\neq 20$. 

Previous results of variance-based global sensitivity analysis show that inputs $x_2, x_{20}, x_{51},$ and $x_{54}$ are distinctly more important than the remaining inputs. Among the remaining inputs, input importance tends to increase as the input index increases, while $x_1,x_3,x_4,x_5,$ and $x_{50}$ also have some importance. The 100D function's expression in~\eqref{equ:100d} explains this. Importance scales with the input index because the summand terms are multiplied by the input index $k$. Inputs  $x_1, x_2,x_3,x_4,x_5, x_{50},x_{51},x_{54}$ stand out in importance because they appear in extra terms outside the sum. The input $x_{20}$ has special importance because its distribution is unique and wider than that of the other inputs. While all 100 inputs appear in the expression, the 100D function still has a low complexity structure because it features only main effects and second-order interactions.

 \begin{figure}[tbp]
      \centering
      \includegraphics[width=.4\linewidth]{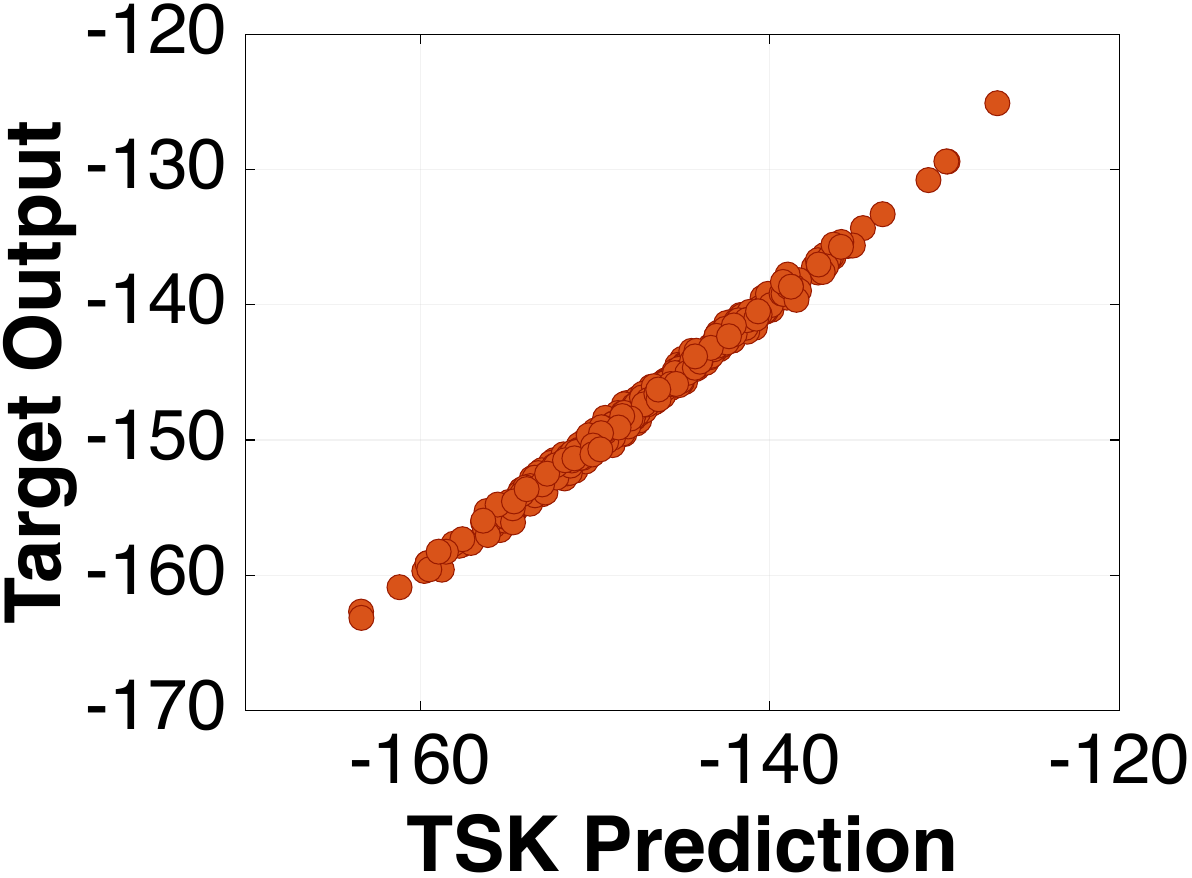}
      \includegraphics[width=.4\linewidth]{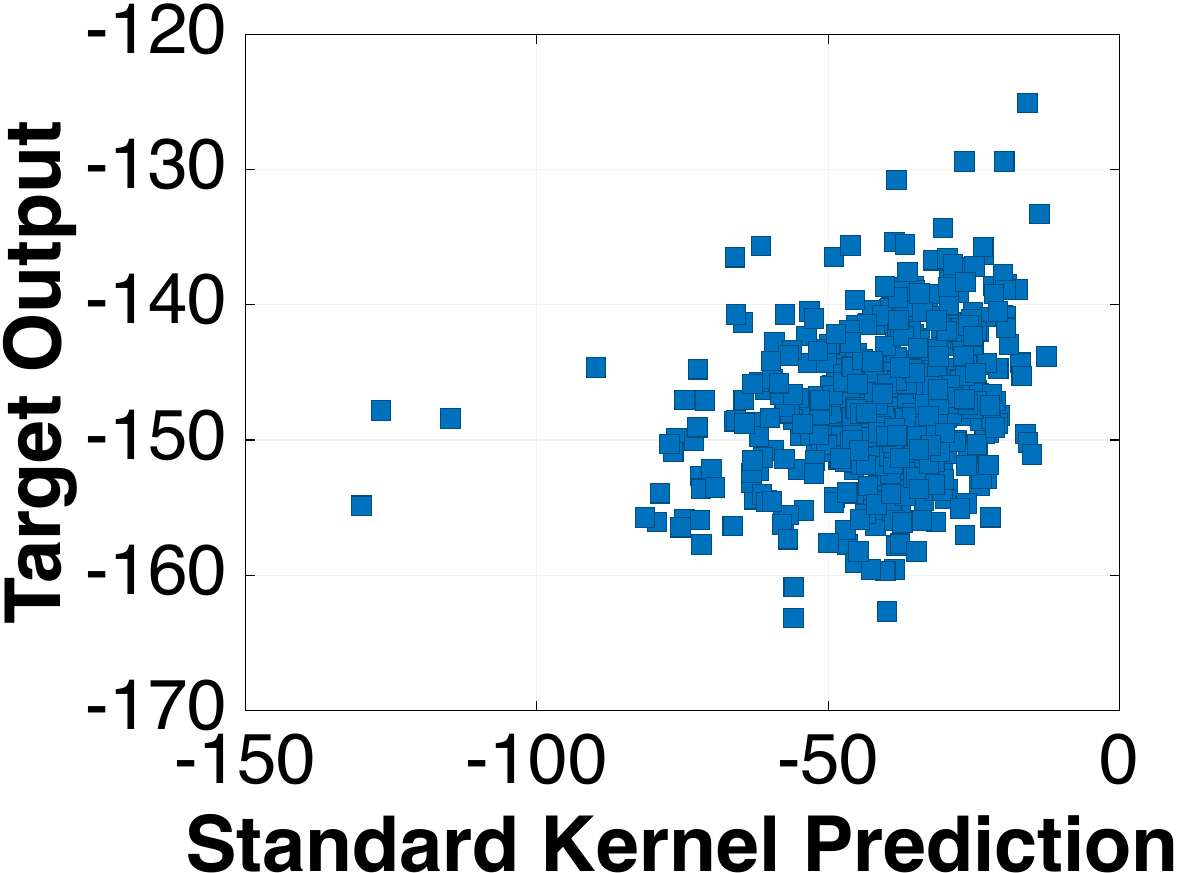}
       \includegraphics[width=.4\linewidth]{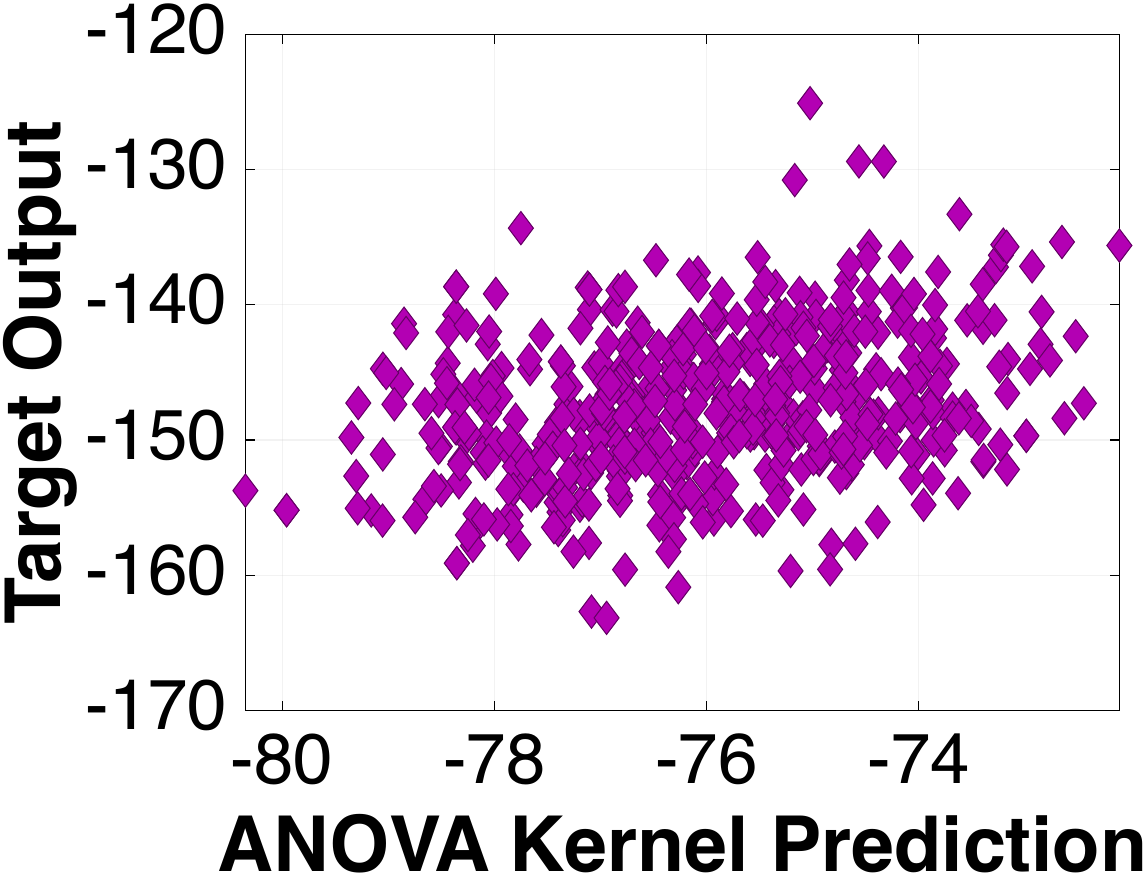}
      \includegraphics[width=.4\linewidth]{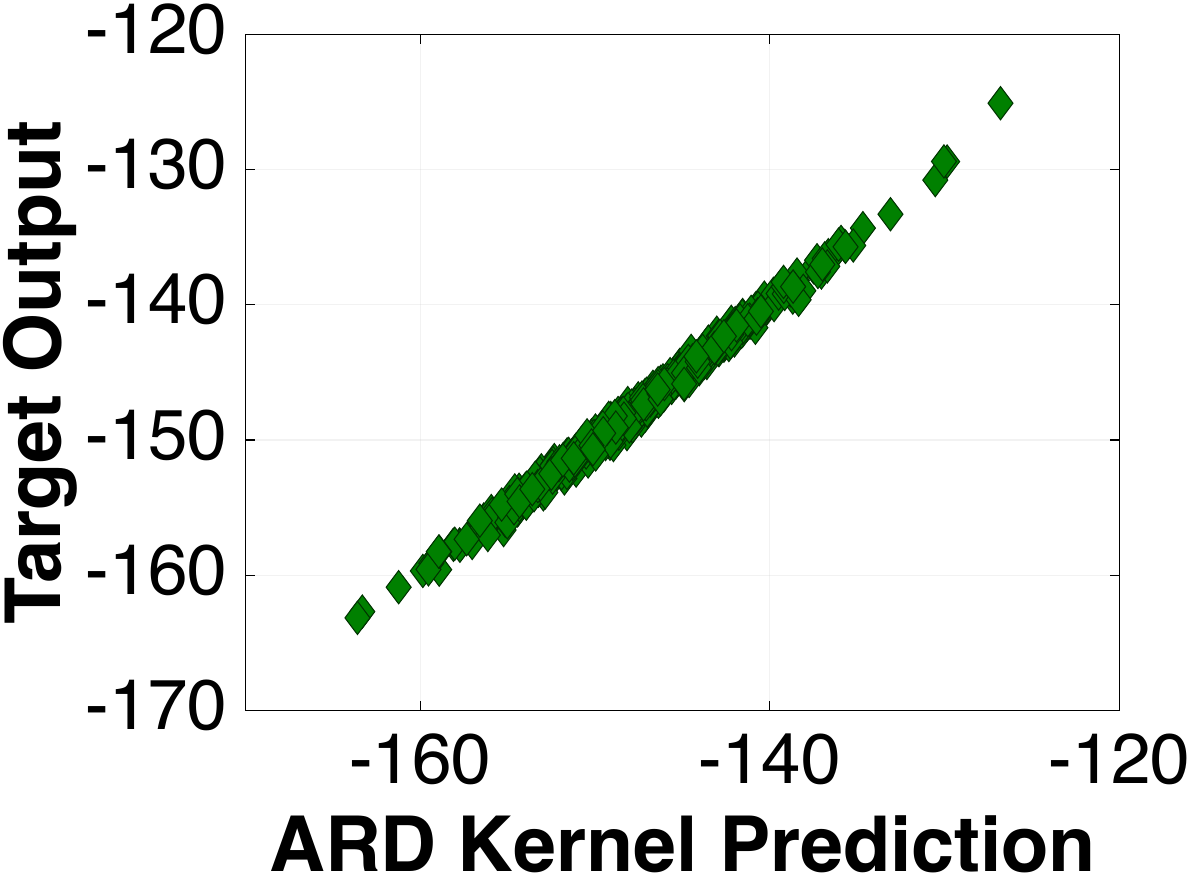}
        \includegraphics[width=.6\linewidth]{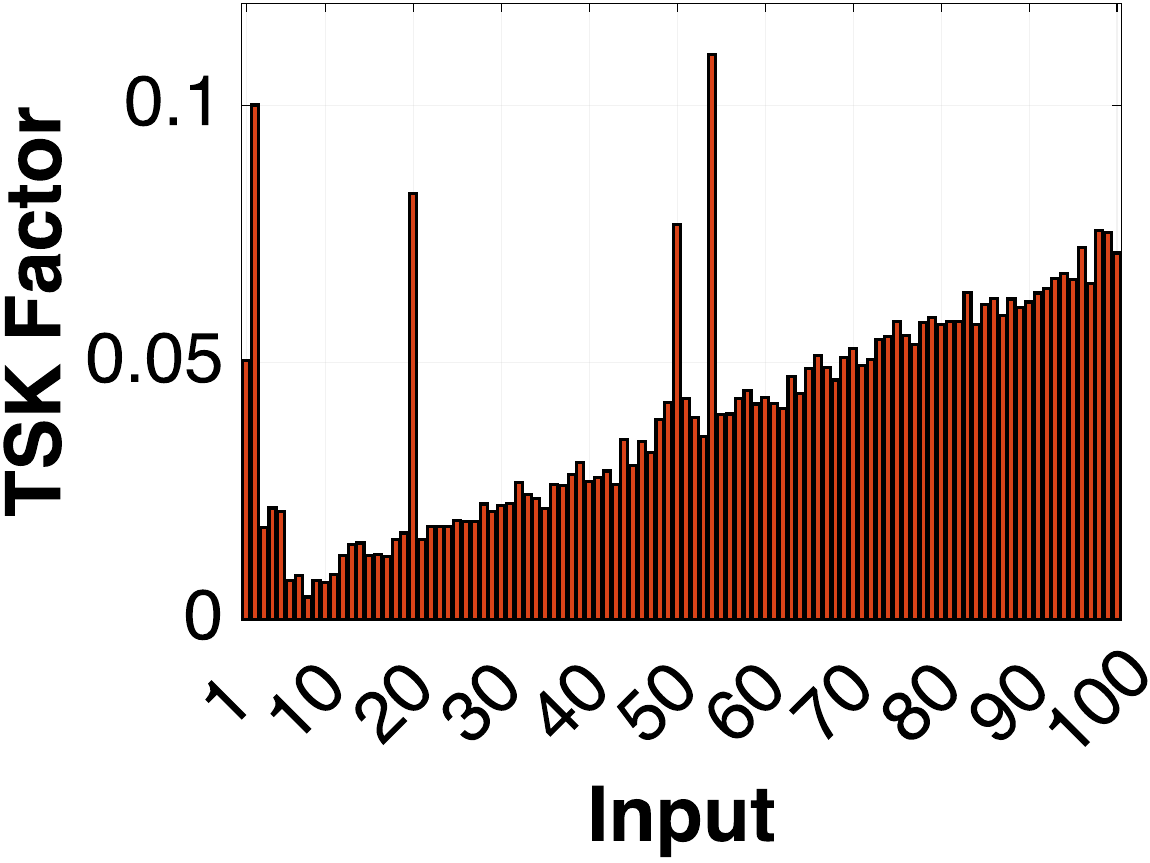}
      \caption{Results for approximation of the 100D function with  TSK, the standard product kernel, the unweighted ANOVA kernel, and ARD, trained on 1000 function evaluations. Top figures display kernel predictions versus true outputs on 500 out of the $10^4$ validation points. Strong correlation, and more accurate prediction, occurs when the scatter plot is closer to a line with slope one. The bottom figure presents the TSK factors taken after optimization.}
      \label{fig:100d}
  \end{figure}
 The scatter plots in~\cref{fig:100d} show how well the kernel predictions correlate with true output values. The errors in~\cref{tab:summ_err} show both TSK and ARD approximate the function accurately. In contrast, the standard kernels and unweighted ANOVA kernels are much less accurate.

 The optimal TSK factors in~\cref{fig:100d} recover meaningful structure. The TSK factors do not match Sobol' indices in value. The ranking of the inputs by the TSK factors matches the ranking given by the Sobol’ indices in~\cite{Luthen21}. This supports that the process of optimizing the TSK factors learns meaningful multivariable structure of the 100D function.

\subsubsection*{One-dimensional diffusion model}
This benchmark is a one-dimensional stochastic diffusion problem. The problem first appeared as a benchmark in~\cite{oneddiffusion} and subsequently appeared in~\cite{Luthen21}. Consider the boundary value problem
\begin{equation}\label{equ:1ddiff}
\left.\begin{array}{ll}
          -\frac{d}{dz}\big(E(z)\frac{du}{dz}(z)\big) + h(z)= 0, & z\in(0,L) , \\
         \quad u(0) = 0,  \\
         \quad \frac{du}{dz}(L) = F.
         \end{array}\right.
\end{equation}
The diffusion term is a lognormal random field $E(z)=\exp\big(\lambda_E+\zeta_E\,G(z)\big)$, where $G(z)$ is a standard normal stationary Gaussian random field. Suppose we wish to determine how $u(L)$, the solution at the other boundary, depends on the realization of $G(z)$. We represent $G(z)$ through its Karhunen-Lo\`{e}ve (KL) expansion
\begin{equation*}
    G(z)=\sum_{k=1}^\infty \sqrt{\alpha_k} \varphi_k(z)\xi_k,\quad \xi_k\sim\mathcal{N}(0,1).
\end{equation*}
    The KL expansion is truncated to its first 62 terms to preserve 99\% of the random field's variance. This transforms the problem into a deterministic, finite-dimensional one. Define the inputs
    \begin{equation*}
        x_k=\xi_k,\quad k=1,\dots,62.
    \end{equation*}
    We aim to construct a surrogate for $f(\bs x)\coloneqq u(L;\bs x)$, the boundary value we get by solving \eqref{equ:1ddiff} for a realization of $\bs x$. 

    The training and validation sets come from an open-source data set created for the benchmark problem~\cite{oneddiff_data}.  The training set uses $M=600$ samples from the standard normal distribution. The $10^4$ validation samples are likewise samples from the standard normal distribution.
 \begin{figure}[tbp]
      \centering
      \includegraphics[width=.4\linewidth]{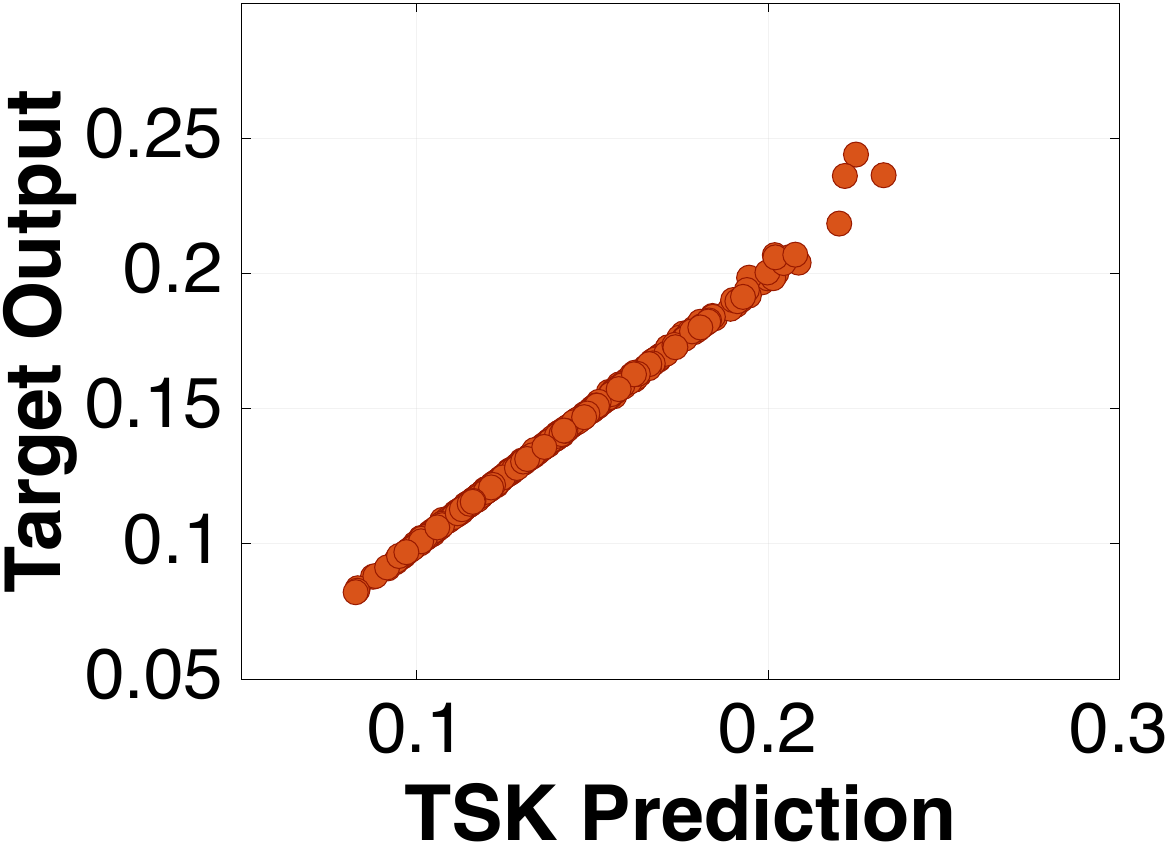}
      \includegraphics[width=.4\linewidth]{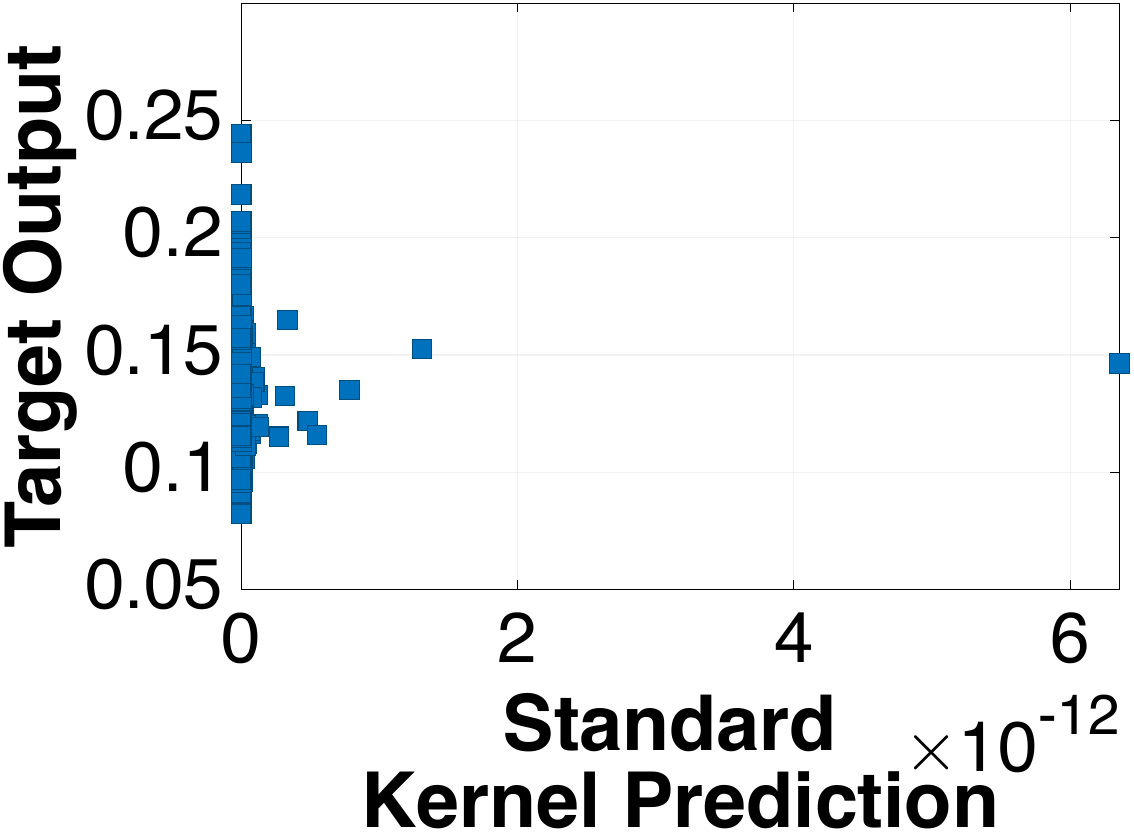}
            \includegraphics[width=.4\linewidth]{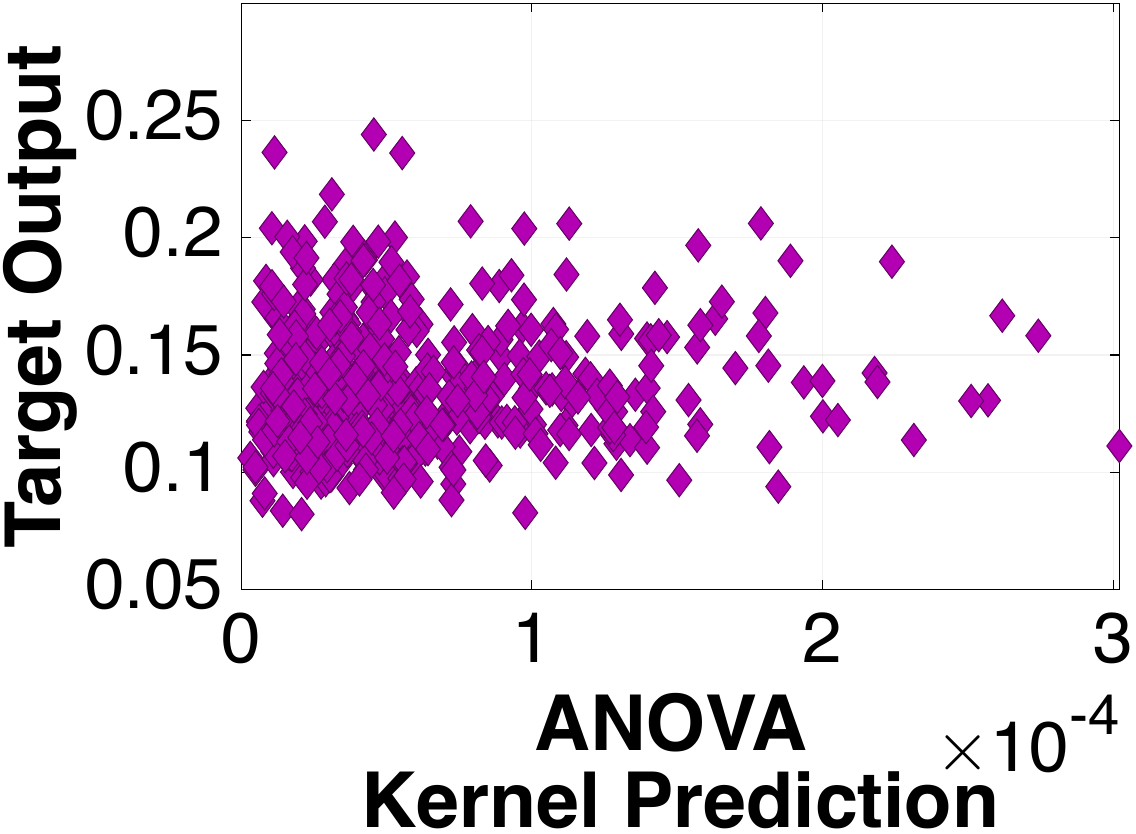}
      \includegraphics[width=.4\linewidth]{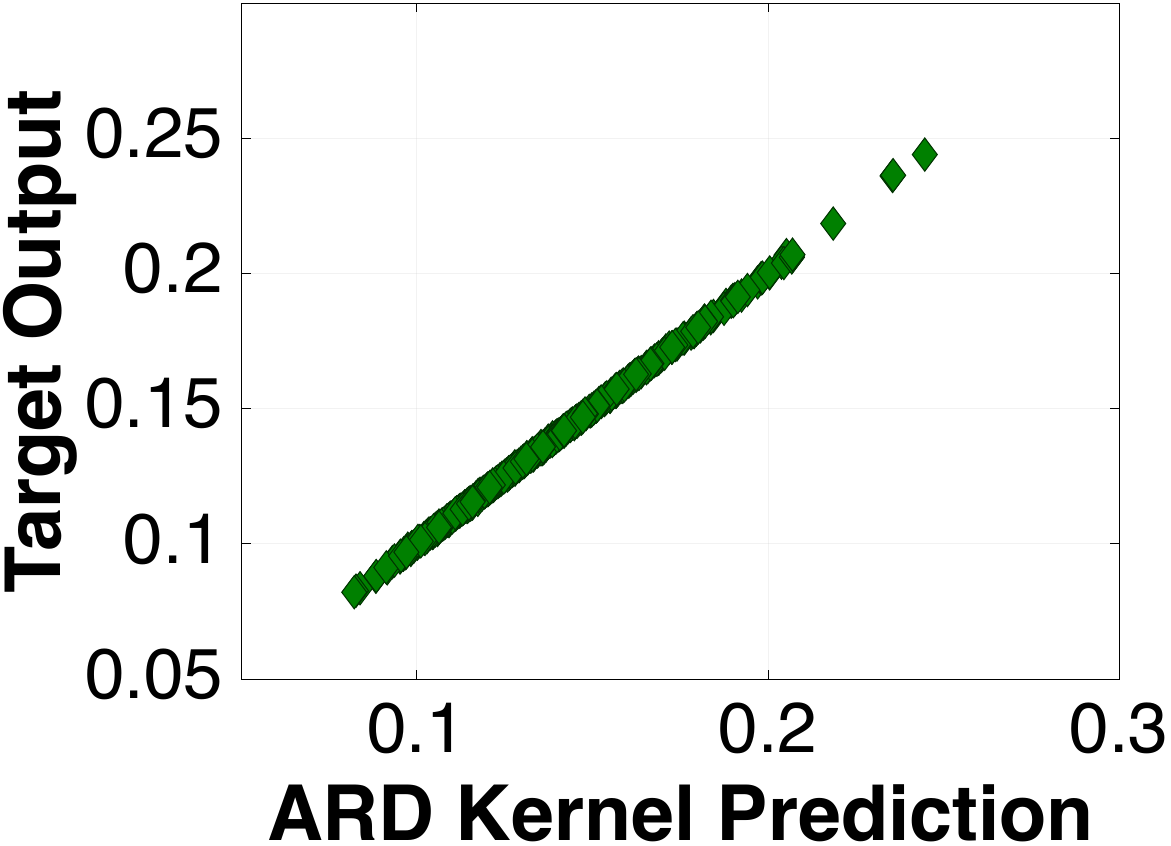}
        \includegraphics[width=.6\linewidth]{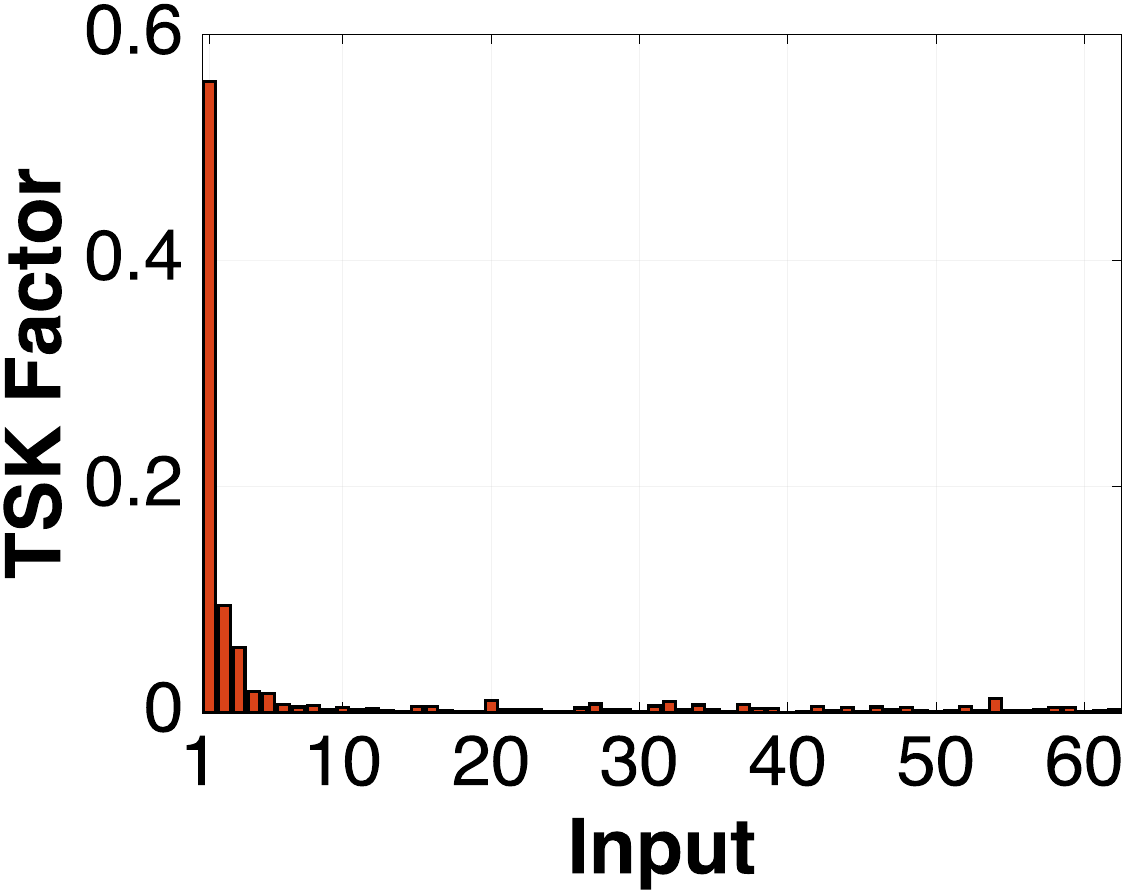}
      \caption{Results for approximation of the one-dimensional diffusion model with  TSK, the standard product kernel, the unweighted ANOVA kernel, and ARD, trained on 600 function evaluations. Top figures display kernel predictions versus true outputs on 500 out of the $10^4$ validation points. Strong correlation, and more accurate prediction, occurs when the scatter plot is closer to a line with slope one. The bottom figure presents the TSK factors $\bs\Sigma^*$ taken after optimization.}
      \label{fig:1d}
  \end{figure}
The results in~\cref{fig:1d}
demonstrate the importance of adapting the kernel to the strongly anisotropic structure of this problem. The standard product kernel gives predictions close to zero and an RRSE of approximately one. The unweighted ANOVA kernel performs similarly, indicating that introducing lower-dimensional ANOVA components without adapting their relative importance is insufficient in this example. Both TSK and ARD produce accurate approximations.

The optimized TSK factors indicate that only a relatively small subset of the $62$ KL coefficients contributes substantially to the quantity of interest. Thus, the improvement of TSK over the two fixed kernels is consistent with its ability to adapt to unequal input importance. 
Variable relevance adaptation through individual kernel length scales provides a particularly effective representation.

\subsubsection*{Modified Grienwank function}
The Grienwank function is an optimization benchmark~\cite{grienwank} structurally composed of polynomial main effect terms and a trigonometric interaction term. We modify the polynomial terms to make the inputs more unequal in importance.
\begin{equation}\label{equ:grienwank}
f(\bs x) = \frac{1}{4000}\sum_{k=1}^{40}\frac{x_k^2}{k^2}    - \prod_{k=1}^{40}\cos\Big(\frac{x_k}{\sqrt{k}}\Big) +1,\quad \bs x\sim\mathcal{U}([-600,600]^{40}).
\end{equation}
In terms of multivariable structure, the importance of the polynomial main effects decreases with the input index. The trigonometric product introduces non-additive dependence across all 40 inputs and, when viewed through an ANOVA decomposition, generates interaction components of multiple orders. The resulting function therefore combines strongly anisotropic input importance with richer interaction structure than the preceding high-dimensional examples. In this respect, it is more similar to the g-function.

 The training set uses $M = 10^3$ samples from the uniform distribution $\mathcal{U}([-600,600]^{40})$. 
 \begin{figure}[tbp]
      \centering
      \includegraphics[width=.4\linewidth]{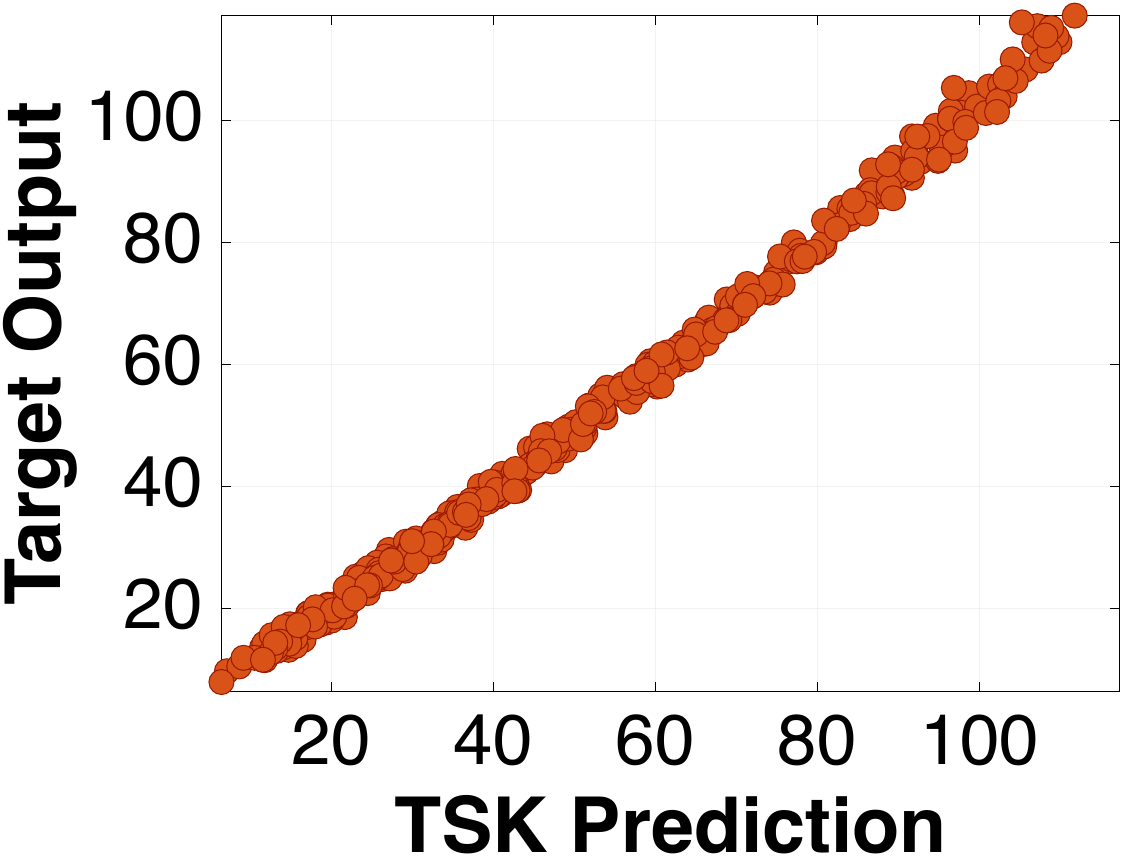}
      \includegraphics[width=.4\linewidth]{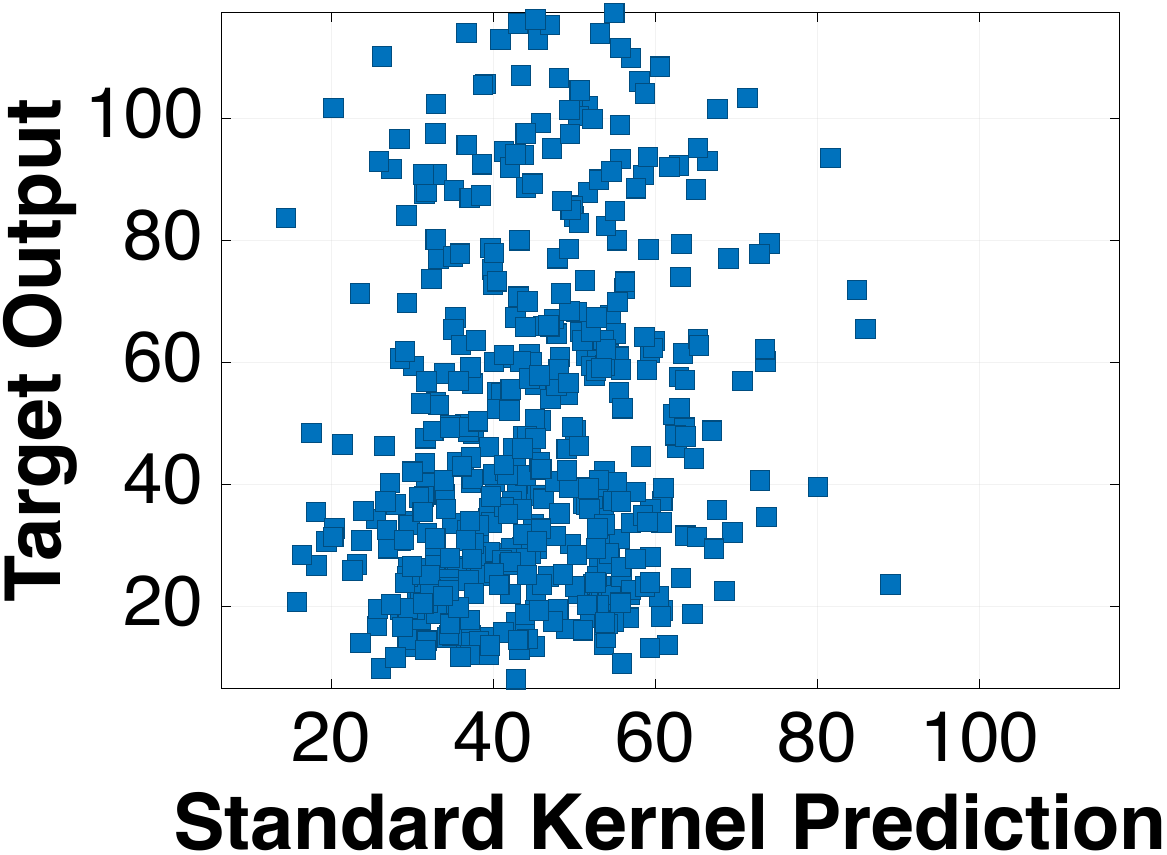}
     \includegraphics[width=.4\linewidth]{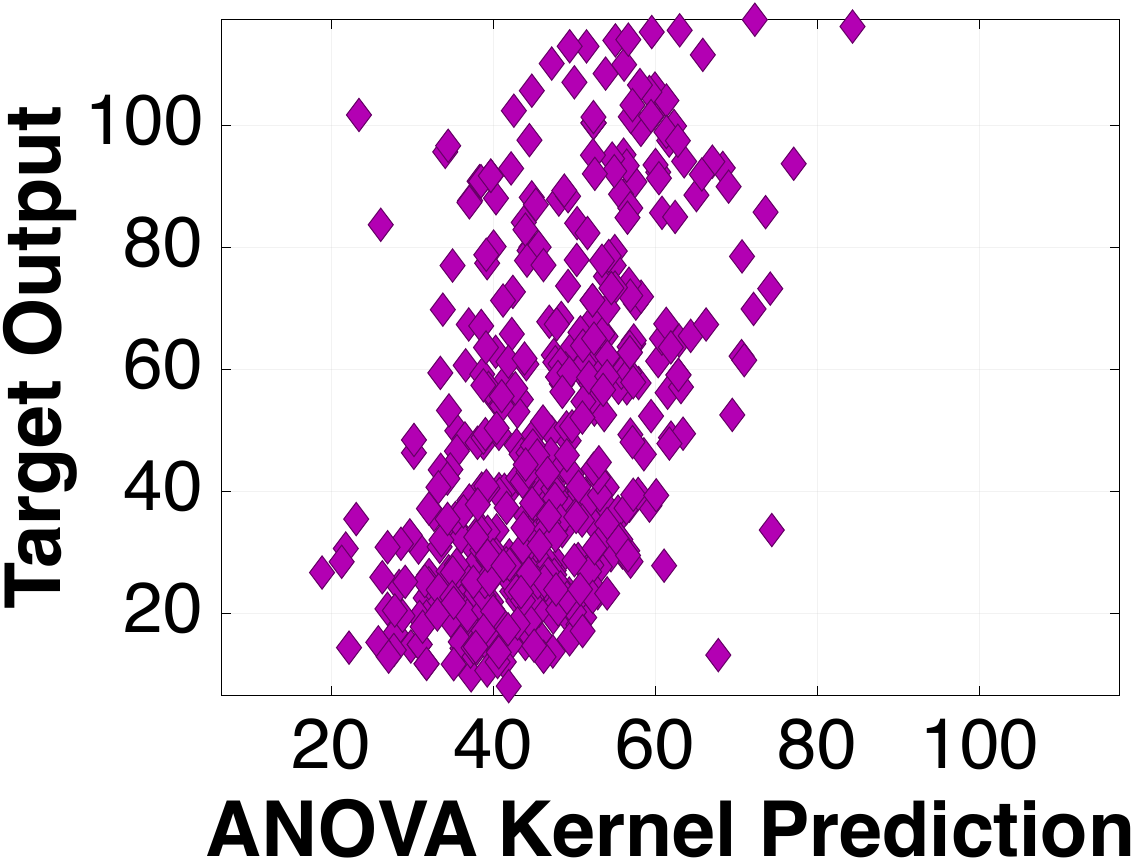}
      \includegraphics[width=.4\linewidth]{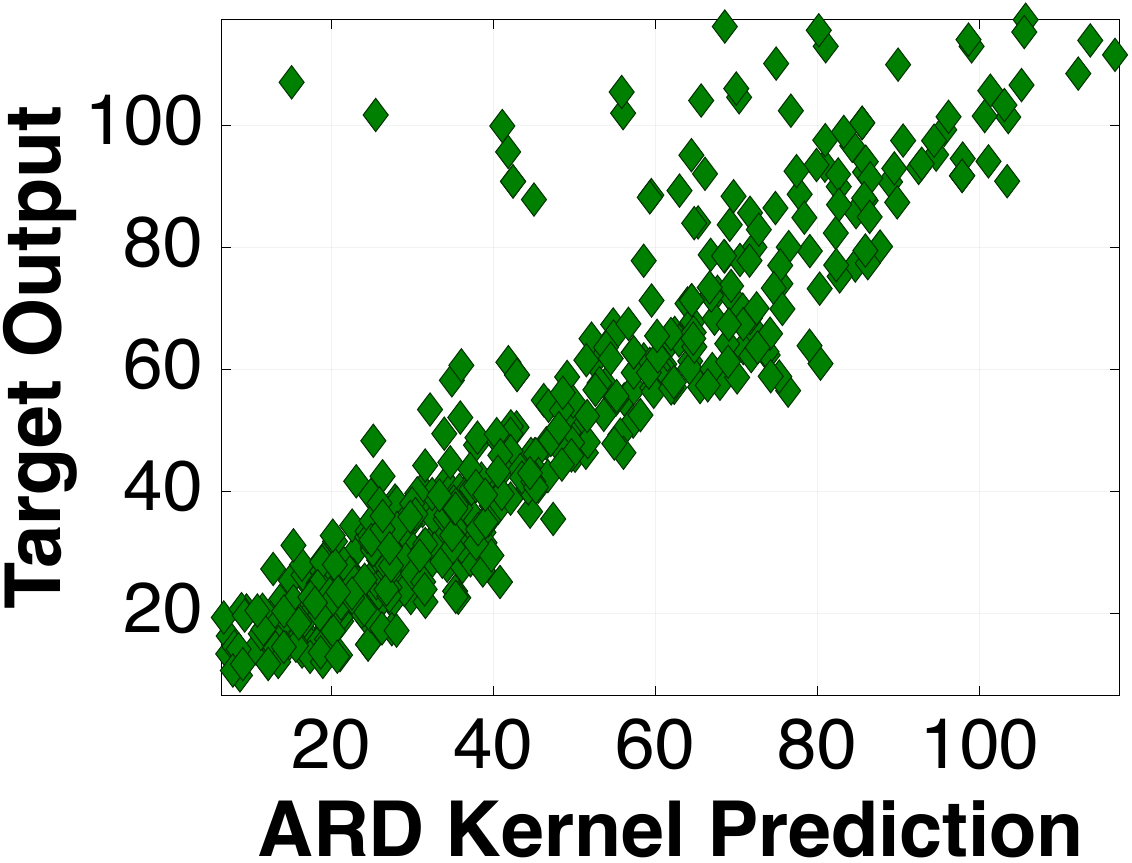}
    \includegraphics[width=.6\linewidth]{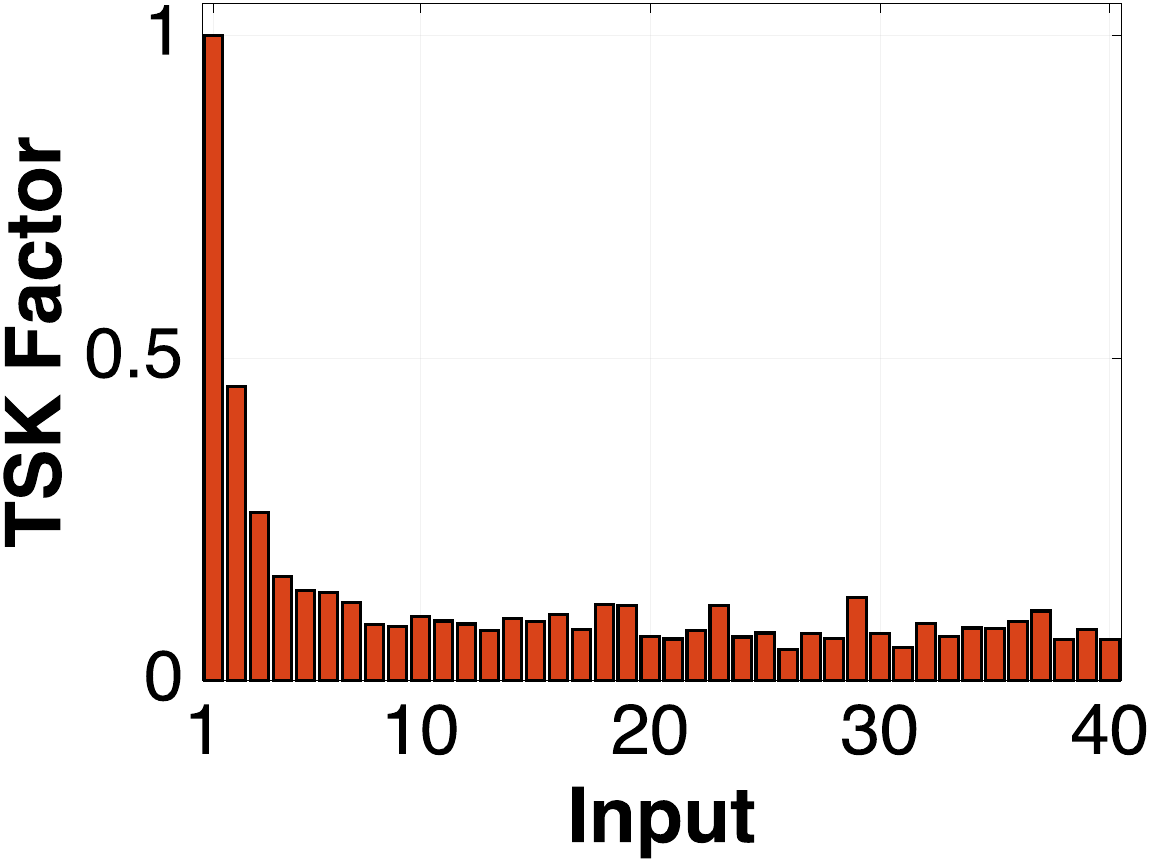}
      \caption{Results for approximation of the modified Grienwank function with  TSK, the standard product kernel, the unweighted ANOVA kernel, and ARD, trained on $M=10^4$ function evaluations. Top figures display kernel predictions versus true outputs on 500 out of the $10^4$ validation points. Strong correlation, and more accurate prediction, occurs when the scatter plot is closer to a line with slope one. The bottom figure presents the TSK factors taken after optimization.}
      \label{fig:grienwank}
  \end{figure}
The modified Grienwank function produces a different comparison between TSK and ARD than the preceding examples. The TSK approximation error is an order of magnitude smaller than the ARD error.\\


\textbf{Summary}\\
The three high-dimensional examples highlight the necessity of identifying multivariable structure and incorporating it into the approximation. The errors in~\cref{tab:summ_err} show that TSK factor optimization can improve performance over the standard kernel by at least two orders of magnitude. The standard product kernel performs poorly in each example, while the unweighted ANOVA kernel provides only modest improvement and performs nearly identically to the standard kernel for the two diffusion problems. Simply introducing an ANOVA decomposition is not sufficient when the contributions of different inputs and interactions are strongly unequal. Both TSK and ARD address this limitation through data-driven kernel adaptation and substantially improve approximation accuracy.

\begin{table}[h!]
    \centering
    \begin{tabular}{|cc|c|c|c|c|} 
    \hline
   Model & Error Type & TSK & Standard Ker. & ANOVA Ker. & ARD \\ \hline
    100D & RMSE & 0.480 & 108.492 & 71.667 & $\mathbf{0.479}$\\
     & RRSE & $3.255\times 10^{-3}$ & 0.736 & 0.486 & $\mathbf{3.251\times 10^{-3}}$ 
      \\ \hline
            1D Diff. & RMSE & ${2.027\times 10^{-3}}$ & ${0.139}$ & ${0.139}$ & $\mathbf{2.755\times 10^{-4}}$ \\
     & RRSE& $1.461\times 10^{-2}$ & 1.000 & 0.999 & $\mathbf{1.985\times 10^{-3}}$  \\ \hline
     Grienwank & RMSE & $\mathbf{1.1814}$ & $29.267$ & $24.643$ & $11.900$ \\
     & RRSE & $\mathbf{3.188\times 10^{-2}}$ & $0.515$ & $0.433$ & $0.209$  \\ \hline
    \end{tabular}
    \caption{Root mean squared error (RMSE) and root relative squared error (RRSE) for approximations of the 100D function, 1D diffusion problem, and modified Grienwank function using TSK, the standard product kernel, the unweighted ANOVA kernel, and the ARD kernel. Errors are taken on validation sets of $10^4$ points.}
    \label{tab:summ_err}
\end{table}

The comparison between TSK and ARD also illustrates that the two methods exploit multivariable structure differently. For the 100D function, which has unequal input importance but only main effects and pairwise interactions, TSK and ARD achieve nearly identical errors. For the diffusion problem, the learned TSK factors indicate that the quantity of interest depends much more strongly on a relatively small subset of the stochastic inputs, and ARD provides the most accurate approximation. In contrast, TSK performs substantially better than ARD for the modified Grienwank function, which combines strongly unequal main effects with interaction components of multiple orders. These results suggest that ARD length scale adaptation and adaptation of ANOVA-component weights provide complementary mechanisms for exploiting multivariable structure. No single adaptive mechanism performs best across all four examples, but TSK provides substantial improvements when its learned weighting of the approximation space aligns with the structure of the target function.

\section{Conclusion and outlook}
This work develops a framework for adapting a weighted RKHS to the multivariable structure of an unknown target function. TSKs provide a structured parameterization of ANOVA kernel weights, reducing $2^d$ possible component weights to $d$ input-specific TSK factors. Rather than prescribing these weights a priori or selecting a sparse collection of components, we learn the TSK factors by seeking the RKHS where the target function has minimum norm. Under suitable conditions, these TSK factors are unique to a target function. The numerical experiments demonstrate that the resulting adaptation improves approximation when the target possesses exploitable multivariable structure.
This perspective on adaptive multivariable approximation argues for adapting the geometry of the approximation space to the structure of the target function. 

Avenues for future work concern the scope and computational implementation of the method. 
The product nature of weights induced by TSK factors restricts the method to certain families of weights. While acceptable for many functions, this is unsuitable for functions with complicated interactions, especially functions that feature higher-order interaction terms but no low-order interaction terms. Quasi-Monte Carlo literature offers many examples of weight structures, including product-order-dependent weights~\cite{KuoSloan12qmc,kaarnioja2021kernel}. Future work should investigate whether TSK theory can extend to other weight structures.
Improved computational scalability for large datasets can be addressed by the use of random features~\cite{RahimiRecht07} instead of using kernels corresponding to the characteristic functions $\phi_k$ directly. The structural adaptability of TSKs makes them attractive for direct integration into established kernel-based machine learning methods like automatic relevance determination~\cite{ard25} and multiple kernel learning~\cite{gonen2011multiple}. While our current theoretical results focus on exact minimum-norm interpolation, extending these uniform error bounds and convergence guarantees to settings with noisy observational data is a crucial next step. In the current work, we establish the method in the approximation of functions with independently distributed inputs. An important future direction is to extend the framework to settings with dependent inputs.
Finally, given the recent exploration of sensitivity analysis beyond the classical ANOVA decomposition setting~\cite{DaV21,Lamboni24,weidensager26,larsen2026new}, it would be interesting to study how TSK factors relate to alternative notions of sensitivity, especially in kernel-based sensitivity indices.

\subsubsection*{Acknowledgments} This work is partially supported by the National Science Foundation (NSF) under grant number 2038118 (John E. Darges). This work was partially funded by the Deutsche Forschungsgemeinschaft (DFG, German Research Foundation) - project number 569580074 (Laura Weidensager).

\bibliographystyle{abbrvurl}
\bibliography{refs}

\end{document}